\documentclass{article}

\usepackage[nonatbib,preprint]{neurips_2023}

\usepackage[utf8]{inputenc} 
\usepackage[T1]{fontenc}    
\usepackage{hyperref}       
\usepackage{url}            
\usepackage{booktabs}       
\usepackage{amsfonts}       
\usepackage{nicefrac}       
\usepackage{microtype}      
\usepackage{xcolor}         
\usepackage{amsmath, amsfonts, amssymb, xspace, color, amsbsy, caption, adjustbox, amsthm}
\usepackage{tabularx}
\usepackage{tikz, graphicx}
\usepackage{enumitem}
\usepackage{multirow}
\usepackage{cleveref}
\usetikzlibrary{positioning}
\usepackage[linesnumbered,ruled,lined]{algorithm2e}
\usepackage[font=small,labelfont=bf]{caption}
\usepackage{subcaption}
\usepackage{bbm}
\usepackage{esvect}
\usepackage{wrapfig}

\newtheorem{prop}{Proposition}
\usepackage{changepage}
\usepackage{pifont}
\newtheorem{definition}{Definition}[section]
\newtheorem{theorem}{Theorem}[section]
\newtheorem{corollary}{Corollary}[theorem]

\newtheorem{thm:eg}{Example}

\newcommand{\Ours}{\textbf{\textsc{Gconda}}\xspace}
\newcommand{\xhdr}[1]{\vspace{1.7mm}\noindent{{\bf #1}}}

\newcommand{\ie}{\emph{i.e.}\xspace} 
\newcommand{\eg}{\emph{e.g.}\xspace} 
\newcommand{\wrt}{\emph{w.r.t.}\xspace} 

\newcommand{\ssym}{§}

\def \E {\mathrm{E}}

\def \d {\mathbf{d}}

\def \f {\mathbf{f}}
\def \g {\mathbf{g}}
\def \h {\mathbf{h}}

\def \w {\mathbf{w}}
\def \x {\mathbf{x}}
\def \y {\mathbf{y}}

\def \P {\mathbf{P}}

\def \W {\mathbf{W}}

\def \D {\mathcal{D}}

\def \G {\mathcal{G}}

\def \N {\mathcal{N}}
\def \O {\mathcal{O}}

\def \S {\mathcal{S}}
\def \T {\mathcal{T}}

\SetKwRepeat{Do}{do}{while}

\newcounter{oldtocdepth}

\newcommand{\hidefromtoc}{%
  \setcounter{oldtocdepth}{\value{tocdepth}}%
  \addtocontents{toc}{\protect\setcounter{tocdepth}{-10}}%
}

\newcommand{\unhidefromtoc}{%
  \addtocontents{toc}{\protect\setcounter{tocdepth}{\value{oldtocdepth}}}%
}

\title{Explaining and Adapting Graph Conditional Shift}

\author{Qi Zhu$^*$ \And
Yizhu Jiao$^*$ \And
Natalia Ponomareva$^\dagger$ \And
Jiawei Han$^*$ \And
Bryan Perozzi$^\dagger$
\AND
\\
*: University of Illinois Urbana-Champaign $\dagger$: Google Research\\
\texttt{*\{qiz3,yizhuj2,hanj\}@illinois.edu}, \\
\texttt{$\dagger$\{nponomareva,bperozzi\}@google.com}
}

\begin{document}

\maketitle

\hidefromtoc

\begin{abstract}
Graph Neural Networks (GNNs) have shown remarkable performance on graph-structured data.
However, recent empirical studies suggest that GNNs are very susceptible to distribution shift. 
There is still significant ambiguity about \emph{why} graph-based models seem more vulnerable to these shifts.
In this work we provide a thorough theoretical analysis on it by quantifying the magnitude of \emph{conditional shift}\footnote{Conditional shift represents a change in the conditional distribution $\P(\y|\x)$ between the input features $\x$ and the corresponding output labels $\y$ when moving from the source domain to the target domain.} between the input features and the output label.
Our findings show that both graph heterophily and model architecture exacerbate conditional shifts, leading to performance degradation. 
To address this, we propose an approach that involves estimating and minimizing the conditional shift for unsupervised domain adaptation on graphs.
In our controlled synthetic experiments, our algorithm demonstrates robustness towards distribution shift, resulting in up to 10\% absolute ROC AUC improvement versus the second-best algorithm. Furthermore, comprehensive experiments on both node classification and graph classification show its robust performance under various distribution shifts.
\end{abstract}
\section{Introduction}
\label{sec:intro}

Graph Neural Networks (GNNs)~\cite{kipf2016semi,velivckovic2017graph,hamilton2017inductive,chami-survey} are powerful tools that have showed excellent performance on graph structured data. 
Interestingly, recent work has revealed that GNNs shows a susceptibility to performance degradation when confronted with data distribution shift, where the data used for training (\textit{source data}) and inference (\textit{target data}) come from different distributions~\cite{koh2021wilds,yehudai2021local}. 
Consequently, there has been a growing interest in investigating the behavior of GNNs under distribution shift, which demonstrate that both shifts in graph structure and node features can lead to a deterioration in GNN performance~\cite{zhu2021shift,ma2021subgroup,udagcn2020}.
However, these works are primarily empirical, and there are still many open questions about both the nature of this susceptibility, as well as how to address it effectively.


At the same time, extensive research has been conducted to examine the behavior of conventional machine learning models (excluding GNNs) in the presence of domain shift. Two prominent settings that have received significant attention are Unsupervised Domain Adaptation (UDA) and Domain Generalization (DG).
When unlabeled target data is available, common UDA approaches including learning Domain Invariant Representation Learning (DIRL) \cite{Ben-David2010, ganin2016domain} attempts to align latent representations of source and target data. Another approach, Domain Generalization (DG) \cite{arjovsky2019invariant}, addresses the challenge of training models that can generalize effectively to unseen target domains by leveraging multiple source domains during training.

In this paper, our focus is on node classification~\cite{kipf2016semi}, a fundamental task in GNNs, where the effectiveness of DIRL methods~\cite{ganin2016domain} has been found to be limited\cite{zhu2021shift}. 
To address this limitation, we begin with investigating the underlying distribution of latent representations generated by GNNs. 
Remarkably, we discover that GNNs can exacerbate the conditional shift ($\P_s(\y|\h)\neq\P_t(\y|\h)$), thereby challenging the validity of ``no conditional shift'' assumption ($\P_s(\y|\h) \approx \P_t(\y|\h)$) in DIRL methods.
In Section~\ref{sec:cond_shift_gnn}, we provide a thorough theoretical analysis by quantifying the magnitude of conditional shift.
Through our investigation into various graph characteristics, we observe that both graph heterophily~\cite{zhu2020beyond} and model architecture (specifically, graph convolutions~\cite{kipf2016semi}) exacerbate conditional shifts.
Our theoretical results then show that these shifts provably degrade the generalization capabilities of GNNs. Thus, mitigating the conditional shift is crucial for enhancing unsupervised domain adaptation on graphs.


Inspired by this understanding, we propose a \textbf{g}raph \textbf{cond}itional shift \textbf{a}daptation method, called \Ours, to perform graph UDA.
First, we estimate the conditional shift as Wasserstein distance $\widehat{\W}_1$ between source label distribution $\P_s(\y|\h)$ and estimated pseudo label distribution $\P_t(\hat{\y}|\h)$. 
Building upon our theoretical results, we incorporate the calculation and minimization of the estimated conditional shift between the source and unlabeled target batch into the training process.
Notably, we enhance our approach by incorporating the distribution discrepancy of the latent representation $(\P_s(\h), \P_t(\h))$ into the estimation of $\widehat{\W}_1$, which we refer to as \Ours++.
In Theorem~\ref{thm_3}, we discuss the generalization bound of GNNs with $\widehat{\W}_1$ and the Lipschitz constant of GNNs~\cite{chuang2022tree,you2023graph}.


Specifically, our theoretical and practical contributions are the following:

\xhdr{(i). Derivation of graph conditional shift and its implications.} Using a CSBM model, we provide the \emph{first provable result} (Theorem~\ref{thm_1}) quantifying how GNNs worsen conditional shift. Subsequent analysis (Corollary~\ref{cor:p-q}) identifies graph heterophily and graph convolutions as two contributing factors to the unsatisfactory performance of GNNs under distribution shifts.
This finding (Corollary~\ref{cor:cond-gap}) further offers insights into the practical implications and applications of GNNs.


\xhdr{(ii). Graph UDA by minimizing conditional shift.} Building upon our theoretical findings, we propose \Ours, a graph UDA method that leverages the minimization of conditional shift.
In practice, we observe a strong correlation between the estimated Wasserstein distance and the actual performance of the GNN model.
In contrast, other latent representation distances that do not exhibit the same level of correlation (\eg CMD~\cite{cmd2017} in Figure~\ref{fig:metric-comparison}), 

\xhdr{(iii). Robustness towards different distribution shifts.}
On synthetic graphs, \Ours demonstrates a substantial performance advantage over other DIRL baselines, with an absolute AUC\_ROC improvement of up to 10\%. In the node classification task, \Ours consistently outperforms competing methods across six real-world datasets, demonstrating superior performance even under various types of shifts. Additionally, when applied to graph classification, our approach leads to performance improvements as well.

\section{Related Work}
\label{sec:related}
\xhdr{Unsupervised Domain Adaptation.} The goal of UDA algorithms is to transfer knowledge from the source onto target data, obtaining good generalization on target distribution. 
In the theoretical foundational work of domain adaptation, \cite{Ben-David2010} presented an upper bound of target risk using the performance of the model on source data and introduced a domain discrepancy measure called $\mathcal{H}$-divergence. Since then, many domain adaptation algorithms that minimize differences between source and target domains have been proposed~\cite{ganin2016domain,cdan2018,gretton2012kernel,mmd2015,cmd2017}. For example, DANN~\cite{ganin2016domain} achieves domain invariant learning (DIRL) by introducing an adversarial objective to distinguish source and target samples in the latent space. Conditional DANN work - CDAN~\cite{cdan2018} - incorporates classifier predictions into the adversarial head, either via linear or multilinear conditioning, further improving UDA performance. Besides, some other work propose to match the distribution in the latent space through probability discrepancy measures like MMD~\cite{gretton2012kernel,mmd2015} and CMD~\cite{cmd2017}.
In a recent study~\cite{zhao2019learning}, it was demonstrated that existing methods for UDA suffer from poor generalization when there is variation in the conditional probability $\P(\y|\x)$ across domains. In response to this challenge, Wasserstein distance on joint \cite{courty2017joint} or label distribution \cite{le2021lamda} are proposed to guide the mapping between source and target samples using optimal transport.


\xhdr{Graph Domain Adaptation.}
Graph Representation Learning introduces new out-of-distribution (OOD) challenges based on the graph structure (including graph size~\cite{bevilacqua2021size,yehudai2021local}, molecular scaffolds\cite{gui2022good}). 
The first several studies\cite{zhang2019dane,udagcn2020,cai2021graph} adopted domain invariant learning across source and target graphs assuming
covariate shift.
On semi-supervised learning,
SRGNN\cite{zhu2021shift} introduced a combination of instance weighting and DIRL techniques to enhance OOD generalization in the presence of localized training data. 
Other pioneering work tried to capture environment-invariant node properties~\cite{eerm2022} and substructures~\cite{yang2022learning} guided by reinforcement learning based environment generators. 
In the meantime, theoretical analysis on the generalization bound of Graph Domain Adaptation (GDA) approaches is advancing. The Tree-mover's distance~\cite{chuang2022tree} provided a model-agnostic generalization bound for GNNs when facing distribution shift. Additionally, the first model-based GDA bound~\cite{you2023graph} proposed to optimize the Lipschitz constant of GNNs through spectral regularization.

Existing domain adaptation algorithms for GNNs primarily focused on enhancing model design to achieve improved empirical performance.
Unlike all these methods, our work introduces a novel perspective - \emph{conditional shift} to explain and mitigate the distribution shift for graph data.

\section{Understanding Distribution Shift in GNNs}
\label{sec:method}

\subsection{Background: Graph UDA}
\label{sec:bkg_gda}
\xhdr{Notations.} A graph is described by a tuple $\G(V, A, X)$, where the nodes $V$ are associated with their features $X \in \mathbb{R}^{|V| \times d}$ and the adjacency matrix $A\in \mathbb{R}^{|V| \times |V|}$ describes the connections between nodes.
We denote $Y$, ($Y \in \mathbb{Z}^{|V| \times |L|})$, as labels for all nodes in graph $\G$ and $x_i ,y_i$ represent a single node's features and label ($y_i \in L$). 
A Graph Neural Network $g$ stacks several neural network layers which transform nodes
and their neighborhood information into a latent representation $g: (X, A) \rightarrow H$. Each layer of a GNN can be described by:
\begin{equation}
     H^k = \sigma(\Tilde{A}H^{k-1}\theta^k),
    \label{eq:GNN}
\end{equation}
where $\Tilde{A}$ is a transformed adjacency matrix that is defined by a specific GNN method. 

The task of node classification takes nodes features $X$ and structure of the graph $A$ to predict labels $Y$ through a GNN encoder $g$ and classifier $f$.
Let the embedding $h_i$ be node $i$'s representation calculated by the final activations of a GNN's output $H$.  Then the task of binary node classification predicts the label using classifier $f$ as follows,
\begin{equation}
    f(\h) = \begin{cases}
    1,& \text{if } \w^T \h +b > 0\\
    -1,              & \text{otherwise}
\end{cases}
\label{eq:lin_cls}
\end{equation}
\textbf{Graph Unsupervised Domain Adaptation.} Given a source and target graph $\G_\S(V^s, A^s, X^s)$ and $\G_\T(V^t, A^t, X^t)$, we assume embeddings $\h^s$ and $\h^t$ are output by the same GNN.
The Unsupervised Domain Adaptation (UDA) algorithm utilizes labeled source $\{(\h^s,\y^s)\}$ data and unlabeled target data $\{\h^t\}$. Let $\varepsilon$ denote the expected risk of a binary classification problem defined above, then UDA aims to find a predictive classifier $f$ and GNN $g$ that achieves small target risk $\varepsilon_\T(f \circ g)$ on $\G_\T$.

To quantify the discrepancy between source and target distributions $\mu_\S$ and $\mu_\T$, we mainly use Wasserstein distance in this paper. In addition, we denote $\mu^f$ as the conditional distribution $\P(\y|\h)$ and $\mu^g$ as the representation distribution $\P(\h)$.
\begin{definition}[Wasserstein distance]
Wasserstein distance is defined between probability distributions $\mu_\S$ and $\mu_\T$ on metric space M, using distance function d, $d: M\times M\rightarrow \mathbb{R}$,
\begin{equation}
    \W_p(\mu_\S, \mu_\T) = \left( \inf_{\gamma \in \Gamma(\mu_\S, \mu_\T)} \E_{(x,y)\sim \gamma} d(x,y)^p  \right)^{1/p},
    \label{eq:w_dist}
\end{equation}
where $p$ is the moment of the distance and $\gamma$ is is a joint probability measure on $M \times M$.
\end{definition}


\xhdr{Domain-Invariant Representations under Covariate Shift.} Covariate shift refers to a change in the distribution of input features (covariates) between the source and target domains. Although labels are unavailable for the target data in UDA setting, DIRL methods~\cite{zhu2021shift,udagcn2020} for GNNs instead optimize the following objective, assuming the covariate shift $\P_s(\y|\h)=\P_t(\y|\h)$,
\begin{equation}
    \min_{f,g} \frac{1}{N} \sum_{i=1}^N \mathcal{L}(y_i^s, h_i^s) + \alpha \W_1(\mu_\S^g, \mu_\S^g),
    \label{eq:dirl-loss}
\end{equation}
where $h_i^s$ is the node representation from GNN's output $H$, $\mu_\S^g:=\P_s(\h)$ and $\mu_\T^g:=\P_t(\h)$ are the marginal distributions of the source and target graphs. The second term minimizes the discrepancy on $H$, which is known as learning a domain invariant representation. 
Besides Wasserstein distance~\cite{shen2018wasserstein}, there are several other notable measures used in DIRL such as CMD~\cite{cmd2017} and MMD~\cite{mmd2015}. 

Below we the give the formal definition of conditional shift. 

\begin{definition}[Conditional Shift]
Assume $\P_s(\y|\x)$ and $\P_t(\y|\x)$ have the same support on $\x$, then conditional shift is defined as $\Delta_{\y|\x}=\int d(\P_s(y|x),\P_t(y|x)) \P_t(x)dx$, $d: L \times L \rightarrow \mathbb{R}^+, y \in L.$
\label{def:conditional_shift}
\end{definition}

\subsection{Conditional Shift in Graph Neural Networks}
\label{sec:cond_shift_gnn}
Now, we present theoretical findings on the occurrence of \emph{conditional shift} in GNNs.
Assuming the conditional shift does take place (\eg covariate shift assumption does not hold), we explore the magnitude of this shift in the input space $\x$ and latent space $\h$ of GNNs. To quantify this shift, we use the terms $\Delta_{\y|\h}$ and $\Delta_{\y|\x}$ to represent the conditional shift in the latent space and input space, respectively. 
To analyze the conditional shift on different graph distributions,  we use the CSBM~\cite{deshpande2018contextual} graph model, an object of recent interest for understanding GNNs~\cite{ma2021homophily,baranwal2021graph}.

\begin{definition}[Contextual Stochastic Block Model (CSBM)] 
\label{def:csbm}
The CSBM graph is a tuple $(A, X, Y)$, where A is the node adjacency matrix, X are the nodes features and Y are the nodes labels $\{y_1, ..., y_n\}$.  These node labels $y_i$ are random variables drawn from a Bernoulli distribution ($\text{Ber}(0.5)$), and control the connections between nodes in the graph. $a_{ij} \sim \text{Ber}(p)$ if $y_i=y_j$ and $a_{ij} \sim \text{Ber}(q)$ otherwise. Features are drawn according to $x_i = y_i \mu  + \frac{Z_i}{\sqrt{d}}$, $y_i \in \{-1, 1\}$, $\mu \in \mathbb{R}^d$ is the feature mean and $Z_i \in \mathbb{R}^d$ is a Gaussian random variable.
\end{definition}

The three parameters of CSBM are $\mu$, $p$, and $q$.  They respectively control the closeness of the two classes, the generated graph's edge density (\eg average degree $D$) and its homophily ratio\footnote{Homophily ratio calculates the fraction of edges in a graph which connects the nodes that have the same label~\cite{zhu2020beyond}.}. 
By manipulating $\mu$ and $(p, q)$, it is possible to generate  distribution shifts of varying magnitude in both node features and graph structure.


To estimate the conditional shift on target CSBM graph $\G_\T$, we define $\Delta_{\y|\x}$ as,
\begin{equation}
\label{eq:cond_shift}
\Delta_{\y|\x} = \mathbb{E}_{\x \sim \P_t(\x)}\left(\mathbb{I}\left[\arg \max_y\P_s(\y|\x)\neq\arg \max_y\P_t(\y|\x) \right]\right),
\end{equation}

\noindent\textbf{Setting:}
The goal of this analysis is to investigate the conditions under which GNNs alleviate such shifts (making covariate shift more likely to hold), or exacerbate them.
Here, we use a 1-layer Graph Convolutional Network\cite{kipf2016semi} as our GNN encoder $g$\footnote{While we present here the results for one-layer GCNs and linear perceptron, our results can be extended to multi-layer graph convolutions with activations in the manner of ~\cite{baranwal2022effects}. We leave this for future work.}. 
On a CSBM graph $(\mu,p,q)$, the means of the two classes in the input space are $(-\mu, \mu)$, while in the latent space they are $(-\frac{p-q}{p+q}\mu, \frac{p-q}{p+q}\mu)$.
Without loss of generality, we assume the distribution shift on feature $\mu$ in $\mathcal{G}_\T$ is controlled by $\delta \in [0,1]$, which moves centroids of both classes in the same direction, that is, $(-(1+\delta)\mu, (1-\delta)\mu)$. 

In Figure~\ref{fig:ood-generalization-roc}, we illustrate how shifts in graph structure and node features can result in conditional shift. When the density or homophily ratio $(D', p'/q')$ changes, the class centroid shifts to different positions, as depicted by $\sqrt{D}\frac{p-q}{p+q}\mu \rightarrow \sqrt{D'}\frac{p'-q'}{p'+q'}\mu$ in Figure~\ref{fig:structure_shift}. Similarly, if the Gaussian mean moves towards a different position (\eg $\mu \rightarrow \mu^\prime$ in Figure~\ref{fig:feature_shift}), it also contributes to the conditional shift. We beging by deriving the conditional shift and expected error in the following theorem:

\begin{figure}[h]
\centering
  \begin{subfigure}{0.23\textwidth}
    \includegraphics[width=1\linewidth]{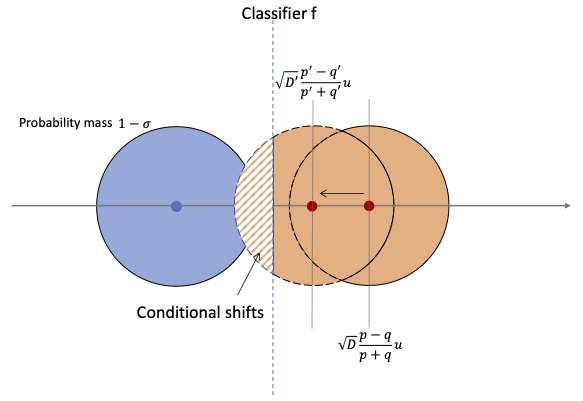}
    \caption{Structure shifts}
    \label{fig:structure_shift}
  \end{subfigure}
\begin{subfigure}{0.23\textwidth}
    \includegraphics[width=1\linewidth]{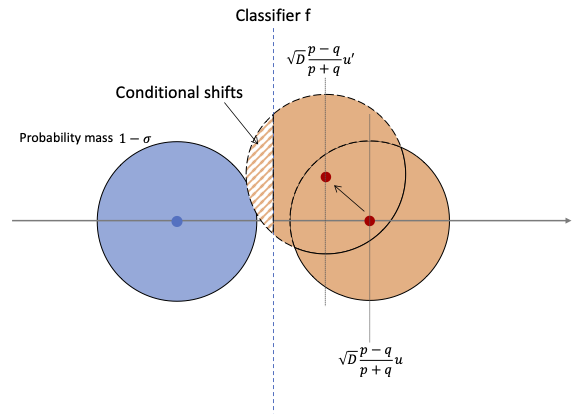}
  \caption{Feature shifts}
  \label{fig:feature_shift}
\end{subfigure}
\begin{subfigure}{0.23\textwidth}
\centering
\includegraphics[width=1\linewidth]{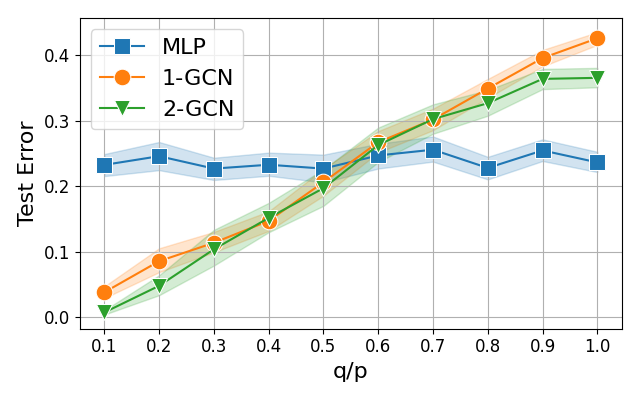}

\caption{Target $\varepsilon_\T$ varying $q/p$}
\label{fig:ood-generalization_a}
\end{subfigure}
\hfill
\begin{subfigure}{0.23\textwidth}
\centering
\includegraphics[width=1\linewidth]{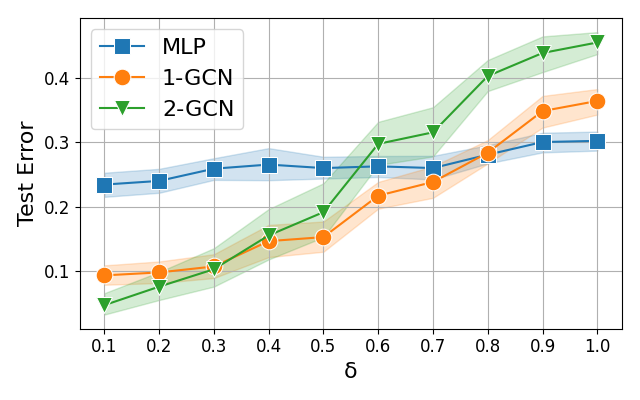}

\caption{Target $\varepsilon_\T$ varying $\delta$}
\label{fig:ood-generalization_b}
\end{subfigure}

\caption{Illustration of the conditional shift in a toy 2D latent space. The red points represent the Gaussian means in the latent space. The x-axis represents the direction of feature means ($\mu$), and the conditional shifts are indicated by the area with hatching lines, as defined in Eq.~\ref{eq:cond_shift}. The corresponding generalization results for both shifts are shown in (c) and (d).}
\label{fig:ood-generalization-roc}
\end{figure}

\begin{theorem}[{Conditional Shift in GNNs}]\label{thm_1}
Let the source graph $\mathcal{G}_\S$ = CSBM($\mu$, $p$, $q$), and a target graph $\mathcal{G}_\T$ = CSBM($\mu'$, $p'$, $q'$), where $D$ and $D'$ represent their average degrees respectively. Additionally, let $\Phi(\cdot)$ denote the cumulative distribution function (CDF) of a multivariate Gaussian distribution defined by distance. 
Then the introduced distribution shift between $\mathcal{G}_\S$ and $\mathcal{G}_\T$ can be quantified via the estimated conditional shift of  $\x$ and $\h$ as:
\begin{equation} \label{eq:gcn_shift}
\Delta_{\y|\x} = \frac{\Phi\left((1+\delta)\| \mu \|) - \Phi((1-\delta)\| \mu\| \right) }{2}, \Delta_{\y|\h} = \frac{\Phi(\|\mu_{h,-1}' \|) - \Phi(\| \mu_{h,1}'\| ) }{2}, 
\end{equation}

where $\mu_{h,1}' = \sqrt{D'}\frac{p'-q'}{p'+q'}\mu - \sqrt{D'} \delta\mu $ and $\mu_{h,-1}' = \sqrt{D'}\frac{q'-p'}{p'+q'}\mu - \sqrt{D'} \delta \mu$.
\end{theorem}
\begin{proof}
See Appendix \ssym A.1. In the proof, we scale the GCN output of the target graph $\h'$ into a standard Gaussian distribution. Then we can compute $\Delta$ by comparing the relative position of the optimal classification hyperplane and mean of the Gaussian. \end{proof}

Supposing two graphs $G_1(\mu,p_1,q_1),G_2(\mu,p_2,q_2)$ have the same feature distribution and edge density, $G_1$ is more heterophilous if it has more edges connecting nodes of different classes, that is $p_1<p_2, q_1>q_2$.
Upon examining the magnitude of the conditional shift $\Delta_{\y|\h}$ in the two graphs, we find that $\mu_{h,-1}'(G_1) > \mu_{h,-1}'(G_2)$ and $\mu_{h,1}'(G_1) < \mu_{h,1}'(G_2)$. This inequality arises due to the fact that $\frac{p_1-q_1}{p_1+q_1} < \frac{p_2-q_2}{p_2+q_2}$. 
In other words, Eq.\ \eqref{eq:gcn_shift} shows that \textbf{heterophilous graphs demonstrates a greater degree of conditional shift}!

\begin{corollary} [GNNs exacerbate Conditional Shift]
\label{cor:p-q}
Assuming only homophily ratio changes $p/q \neq p'/q'$, the conditional shift is always exacerbated by the 1-layer GCN since $\Delta_{\y|\x}=0$.
When there is only a feature shift $\delta \mu$, the shift will be amplified by the GCN as $\sqrt{D}\delta\mu$, potentially leading to larger conditional shifts.

\end{corollary}
Here, $\Delta_{\y|\x}$ is the conditional shift of a non-graph model (e.g.\ a multilayer perceptron) and $\Delta_{\y|\h}$ is the conditional shift of a GCN.
Interestingly, Eq.\ \eqref{eq:gcn_shift} shows GCNs introduce a factor of $\sqrt{D'}$, where $D' > 1$ for any connected graphs. 
In other words, \textbf{GNNs amplify feature shift} (by $\sqrt{D'})$.

\begin{corollary} [Relation between conditional shift and generalization]
\label{cor:cond-gap}
Conditional shift upper bounds the performance gap between source and target, \ie $\Delta > |\varepsilon_\T-\varepsilon_\S|$.
The expected target error $\varepsilon_{\T}$ for linear classifiers $f$ and GNNs $f\circ g$ in section~\ref{sec:bkg_gda} are,
\begin{equation}
    \varepsilon_{\T}(f)= 1 -  \frac{\Phi((1+\delta)\| \mu \| ) + \Phi((1-\delta)\| \mu \|)}{2}, \varepsilon_{\T}(f\circ g)=1-\frac{\Phi(\|\mu_{h,-1}' \|) + \Phi(\| \mu_{h,1}'\| ) }{2}.
\end{equation}
\end{corollary}

Together with the Corollary~\ref{cor:p-q}, we aim to validate the correlation between conditional shift and target error $\varepsilon_\T$.
Therefore, we trained an MLP, a 1-layer GCN, and a 2-layer GCN on a source CSBM graph where GNNs achieves smaller ($\varepsilon_\S \approx 0$) than MLP ($\varepsilon_\S \approx 0.2$). During testing, we kept the graph density unchanged (\textit{i.e.,} $D=D'$), while increasing the heterophily ratio $q'/p'$ or the deviation in feature mean $\delta$.
As shown in Figure~\ref{fig:ood-generalization_a} and Figure~\ref{fig:ood-generalization_b}, we observed (1) GCNs cannot separate the training data more accurately than MLP when shift is large (\ie, a larger  $\varepsilon_\T$); (2) the performance gap $\varepsilon_\T - \varepsilon_\S$ between the source and target is more pronounced in GCNs, confirming that \textbf{conditional shift of GNNs leads to a larger performance drop}.
Having demonstrated that graph inductive bias often exacerbates conditional shift, our focus now turns to exploring potential mitigations of such shift during GNN training.

\section{Graph UDA by Minimizing Conditional Shift}
\label{sec:graph-ot-uda}
In the previous section, we discussed the exacerbated conditional shift for GNNs and how they relate to the performance degradation.
Now, we present our approach to mitigate this conditional shift, quantified using the Wasserstein distance, in order to achieve effective graph UDA.

%

We first introduce the formal definition of optimal transport used in Eq.\ \eqref{eq:w_dist}.
Wasserstein distance can be computed as the optimal transport (OT) cost~\cite{monge1781memoire} between two distributions. Let $\d(u_i,v_j)$ be the distance between two sets of samples $\{u_i\}_{i=1}^m$ and $\{v_j\}_{j=1}^m$ drawn from $\mu_\S$ and $\mu_\T$ respectively. OT solves the following problem: 
\begin{equation}
    \gamma^* = \arg \min_{\gamma \in \Gamma(\mu_s, \mu_t)} \sum_{i,j} \d(u_i,v_j) \gamma(i,j)
\end{equation}
Specifically, $\Gamma$ is the set of transportation plans that satisfy $\Gamma(\mu_s, \mu_t)=\{\gamma \in \mathbb{R}_{+}^{m \times m} | \gamma\mathbbm{1}_m=\gamma^\intercal\mathbbm{1}_m=\mathbbm{1}_m \}$.

To estimate the empirical conditional shift, we calculate the Wasserstein distance between source label $\y^s$ and estimated target label $\hat{\y}^t=f(g(\x^t))$ as $\widehat{\W}_1(\mu_\S^f, \mu_\T^f)$. Hereby we introduce the learning problem of unsupervised graph domain adaptation by minimizing conditional shift. Given source labeled data $\{(x_i^s, y_i^s)\}_{i=1}^N$ in $\G^s$ and unlabeled target data $\{x_j^t\}_{j=1}^N$ in $\G^t$, we optimize the following loss function,
\begin{align}
\label{eq:gjdot_loss}
\mathcal{L}_{\Ours} = \frac{1}{N} \sum_i  \mathcal{L}_\text{CE}(y^s_i, \hat{y}^s_i) + \lambda \widehat{\W}_1(\mu_\S^f, \mu_\T^f), \\
\widehat{\W}_1(\mu_\S^f, \mu_\T^f) = \sum_{ij}\gamma^*_{ij} \cdot \d\left(y_i^s, \hat{y}_j^t\right), \d(\cdot) = \mathcal{L}_\text{CE}(\cdot)\
\end{align}
where $\hat{y}_i$ and $\hat{y}_j$ are predictions on the source and target data produced by the classifier $f$ and GNN encoder $g$. $\mathcal{L}_\text{CE}$ is the cross-entropy loss.
The loss consists of (1) classification loss on $\G^s$; (2) estimated conditional shift $\widehat{W}_1$ between source and target samples in the batch; $\Gamma^* \in \mathbb{R}^{N\times N}$ is the optimal transportation plan between node $i$ in source graph $\G^s$ and j in target $\G^t$, $\sum_{ij}\gamma^*_{ij}=1$. 

Besides matching the conditional distribution $\P(\y|\h)$, we propose to also mitigate the discrepancy marginal probability $\P(\h)$ following ideas from non-GNN UDA works ~\cite{courty2017joint} and call this variant \Ours++.
We define the distance between source data $(h_i^s,y_i^s)$ and target data $(h_j^t,\hat{y}_j^t)$ as,
\begin{equation}
    \d\left((h_i^s,y_i^s), (h_j^t,\hat{y}_j^t)\right) = \alpha \|h^s_i-h^t_j\|^2 + \beta \mathcal{L}_\text{CE}(y^s_i, \hat{y}^t_j)
    \label{eq:pairwise-distance}
\end{equation}
where $h_i=g(x_i)$ is the output of a GNN. \Ours++ optimizes \textit{both} the conditional and marginal distribution, that is, $\widehat{\W}_1(\mu_\S^{f\circ g}, \mu_\T^{f\circ g})$. If we set $\beta=0$, our approach is equivalent to a DIRL method using optimal transport. In our experiments (\ie, Table~\ref{tab:semi-supervised}), we confirm this by showing that \Ours with $\beta=0$ yields similar results to DIRL baselines.



\xhdr{Generalization Bound of \Ours.}
Next we show the relationship between the estimated conditional shift and the generalization error  under distribution shifts. We achieve this by extending theoretical results from \cite{courty2017joint}.

\begin{theorem}
\label{thm_3}
Suppose $\mathcal{F}$ is the hypothesis space of GNNs, $\forall f \in \mathcal{F}$,
\begin{equation}
\varepsilon_\T(f) \leq \varepsilon_\S(f) + \W_1(\mu_\S^{f}, \mu_\T^{f}) + \lambda^* + K_\mathcal{L} K_g \phi(c),
\end{equation}
where $\lambda^*$ is the joint optimal error, $K_\mathcal{L}$ is the Lipschitz constant loss function of loss function $\mathcal{L}$, $K_g$ is the Lipschitz constant of GNN $g$ and $\phi(c)$ is the probabilistic lipschitzness~\cite{ben2014domain}.
\end{theorem}

\proofname. See Appendix \ssym A.2. Assuming a model can generalize well on source and target data (\ie small $\lambda^*$), one can estimate the expected target error through OT cost $\widehat{W}_1$ and the Lipschitz constant $K_g$ of the GNN function. Furthermore, if practitioners aim to improve the generalization on target domain, they can either (1) employ an UDA algorithm (\eg \Ours) to minimize $\widehat{W}_1$ or (2) change the GNN architecture to the one with a smaller $K_g$ Lipschitz constant suggested by recent studies~\cite{chuang2022tree,you2023graph}.


Note that the transportation cost term $\widehat{W}_1$ in our loss function $\mathcal{L}_\Ours$ is an empirical estimation of $\W_1(\mu_\S^{f}, \mu_\T^{f})$ in the bound. To examine that whether the transportation cost is a good domain adaptation metric, we train a 2-layer graph convolution networks on $\mathcal{G}_\S$ and compute $\widehat{W}_1$ on $\mathcal{G}_\T$. The results are presented in Figure~\ref{fig:metric-comparison}. Compared with CMD, $\widehat{W}_1$ demonstrates a more clear correlation between discrepancy and testing performance on both synthetic graphs (\ie CSBM) and real graphs (\ie PubMed) when distribution shifts are present. 

\begin{figure} 
\centering

\begin{subfigure}[b]{0.23\textwidth}
\centering
\includegraphics[width=\textwidth]{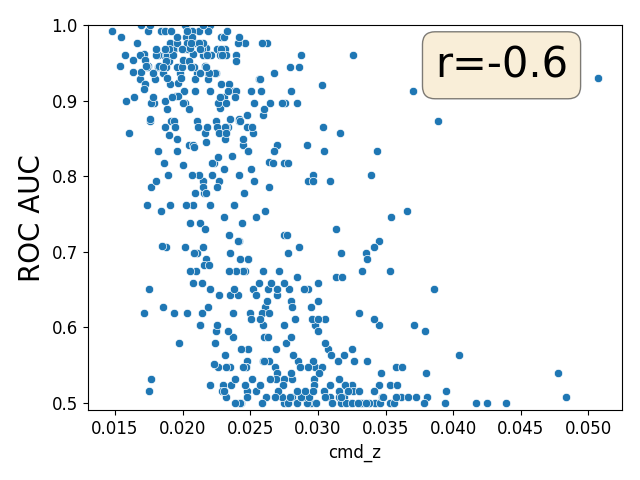}
\caption{CMD on CSBM}
\label{fig:cmd_csbm}
\end{subfigure}
\begin{subfigure}[b]{0.23\textwidth}
\centering
\includegraphics[width=\textwidth]{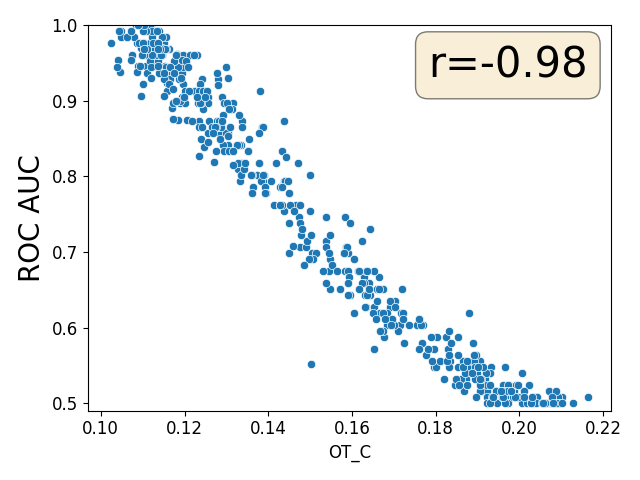}
\caption{$\widehat{W}_1$ on CSBM}
\label{fig:ot_cost_citeseer}
\end{subfigure}
\begin{subfigure}[b]{0.23\textwidth}
\includegraphics[width=\textwidth]{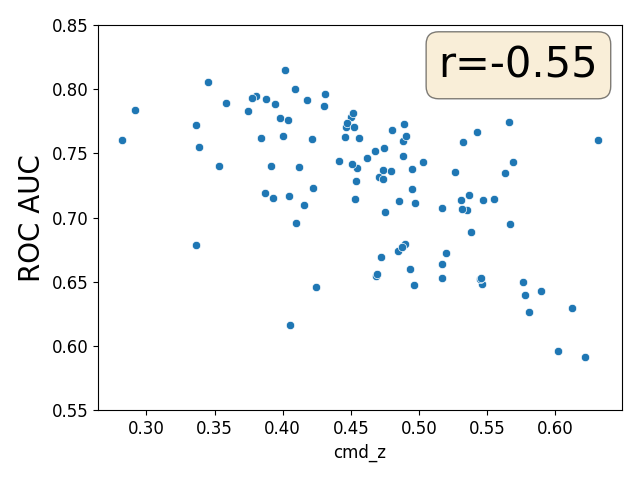}
\caption{CMD on PubMed}
\label{fig:cmd_pubmed}
\end{subfigure}
\begin{subfigure}[b]{0.23\textwidth}
\includegraphics[width=\textwidth]{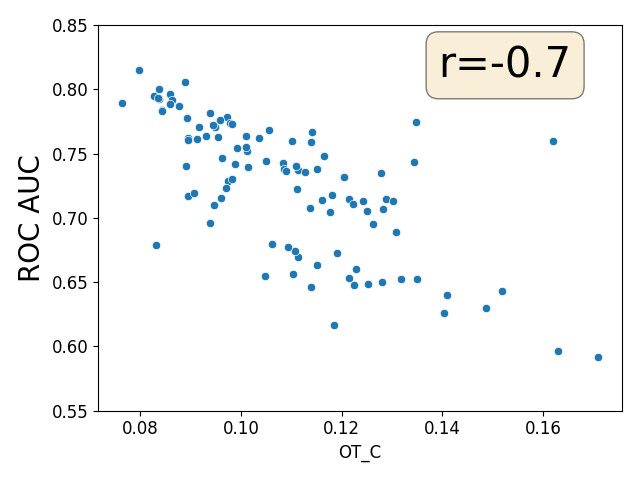}
\caption{$\widehat{W}_1$ on PubMed}
\label{fig:ot_cost_pubmed}
\end{subfigure}

\caption{Comparison of CMD and transportation cost $\widehat{W}_1$ (ours) for the same GCN model, x-axis the metric value and y-axis is the test ROC AUC. 
Each point in the plot corresponds to a pair of source and target graph. The Person correlation coefficient $r$ (the larger abs values, the better) is shown on the top-right corner.}
\label{fig:metric-comparison}

\end{figure}


\xhdr{Optimization.}
We first fix the parameters of GNN $g$ and classifier $f$ to solve the transportation plan $\Gamma^*$ using an EMD solver~\cite{bonneel2011displacement}. Then, we update the parameters of $\{f,g\}$ through back-propagation of $\mathcal{L}_\text{\Ours}$. It is also possible to update our parameters end-to-end with a neural optimal transport solver~\cite{korotin2021neural}. We perform scalable neighborhood sampling on the graph $\G$ to obtain source and target subgraph samples for the input of the GNN $g$. Specifically, we adopt a sub-graph based sampling method - GraphSAINT~\cite{zeng2019graphsaint} to obtain batch of nodes from source and target $\G^s_b \sim \G_\S, \G^t_b \sim \G_\S$, respectively. Refer to Appendix \ssym B.1 for the \Ours algorithm outline.

\xhdr{Complexity.} 
In each step, let $N$ be the size of mini-batch and $d$ be the dimension size of hidden representation $h_i \in \mathbb{R}^d$ and L classes, the additional computation cost of our method  in each epoch is due to computing the transportation cost matrix $C \in \mathbb{R}^{N \times N}$ and solving the optimal transportation $\gamma^*$. The cost matrix takes $\O(N^2(d+L))$ time and the EMD solver takes $\O(N^2)$ to solve the optimal transportation plan. Therefore, the total time complexity of \Ours is $\O(N^2+N^2(d+L))$. Due to the space limit, we conduct experiments on hyperparameter sensitivity and complexity study in Appendix \ssym C.4.



\section{Synthetic Experiments}
\label{sec:syn-experiment}
In this section, we empirically validate our theoretical insights regarding the generalization ability and transferability of graph neural networks. We aim to answer the following questions: (a) "How do  DIRL methods perform under distribution shift on graphs?" and (b) "Does \Ours provide any advantages over DIRL for GNNs?"

We do this using two different families of synthetic graphs: 
(1) CSBM graphs, specifically \texttt{syn-csbm-pq} and \texttt{syn-csbm-$\delta$}, which involve synthetic conditional shifts in both the features and structure. Each sample in the CSBM graph consists of a training and testing graph, where the testing graph demonstrates either a feature shift $\delta$ or a structure shift pq. (2) synthetic graphs constructed from real datasets, namely \texttt{syn-cora} and \texttt{syn-products}, with varying homophily ratios as described in previous work~\cite{zhu2020beyond}.
Detailed numerical results for all of the figures and the graph statistics can be found in Appendix \ssym C.1. 
In this section, we compare our method \Ours with well-known DIRL algorithms including CMD~\cite{cmd2017} and CDAN~\cite{cdan2018} using graph convolution networks~\cite{kipf2016semi}. 


\begin{figure}[h]
\centering
\begin{subfigure}{0.24\textwidth}
\centering
\includegraphics[width=\textwidth]{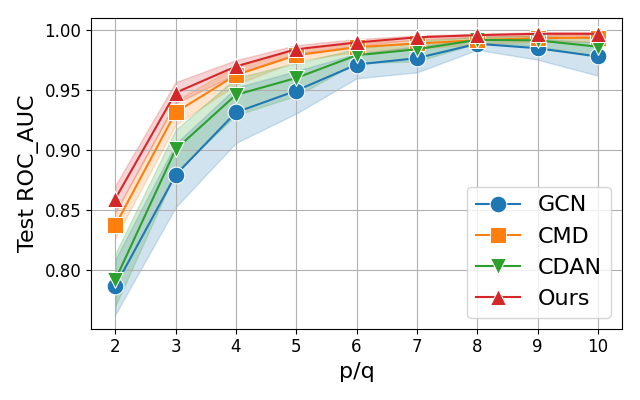}
\caption{\texttt{syn-csbm-pq}}
\label{fig:csbm_pq_adaptation}
\end{subfigure}
\hspace*{\fill}
\begin{subfigure}{0.24\textwidth}
\centering
\includegraphics[width=\textwidth]{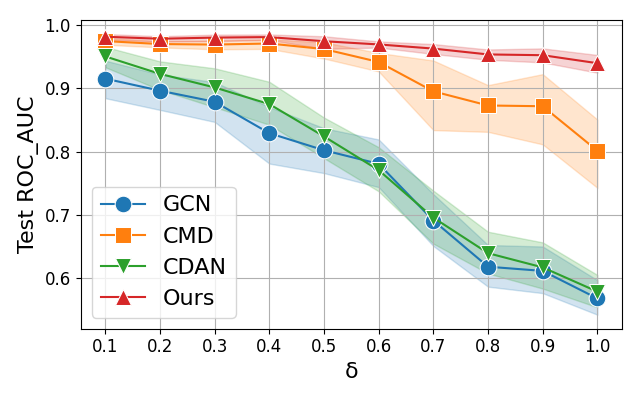}
\caption{\texttt{syn-csbm-$\delta$}}
\label{fig:csbm_delta_adaptation}
\end{subfigure}
\begin{subfigure}{0.24\textwidth}
\includegraphics[width=\textwidth]{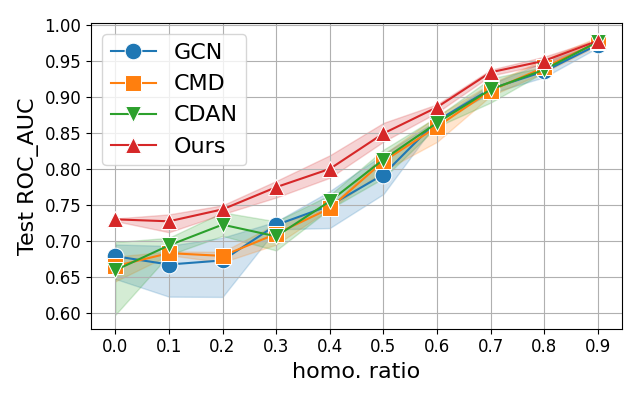}
\caption{\texttt{syn-cora}}
\label{fig:cora_pq_adaptation}
\end{subfigure}
\begin{subfigure}{0.24\textwidth}
\centering
\includegraphics[width=\textwidth]{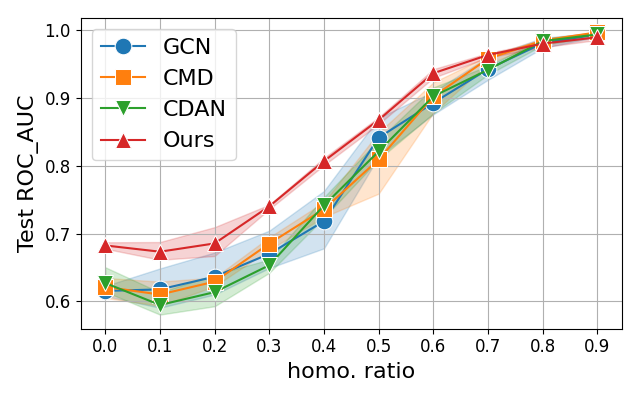}
\caption{\texttt{syn-products}}
\label{fig:products_pq_adaptation}
\end{subfigure}

\caption{Domain adaptation results on synthetic graphs.}
\label{fig:csbm_adaptation}
\vspace{-0.2em}
\end{figure}

First, we compare two DIRL algorithms - CDAN and CMD with \Ours on two synthetic CSBM datasets. We tune the hyperparameters of all three algorithms using validation data obtained from the training graph.
As illustrated in Figure~\ref{fig:csbm_pq_adaptation} and Figure~\ref{fig:csbm_delta_adaptation}, \Ours outperforms both baselines in the presence of feature and structure distribution shifts. Notably, when the tested graph exhibits increased heterophily (small $p/q$) or significant feature shifts (large $\delta$), the performance of GNNs is more adversely affected.
When distribution shifts are small, CMD enjoys similar to our method performance, confirming that DIRL methods work well with minor conditional shift.
However, the difference between two methods becomes significant when testing distribution exhibits large conditional shift. 
We attribute the sub-optimal performance of DIRL to the fact that it solely optimizes the distribution discrepancy on $\P(\h)$ while neglecting the significant conditional shift.




In our second synthetic experiments, we examine the effectiveness of \Ours on non-CSBM graphs. 
To do this, we follow the literature~\cite{zhu2020beyond} studying GNNs in the low homophily setting, where
\texttt{syn-cora} and \texttt{syn-products} are constructed from existing benchmarks via preferential attachment~\cite{barabasi1999emergence}.
We train all the compared methods on the same "easy" graph, which has a homophily ratio of 1.0, for both datasets. Subsequently, we tested the models on target graphs with varying homophily ratios, ranging from 0.0 to 0.9.
Based on our theoretical results, a target graph with a low homophily ratio is expected to result in a larger conditional shift. As depicted in Figure~\ref{fig:cora_pq_adaptation} and Figure~\ref{fig:products_pq_adaptation}, we observe that the performance of the base GCN aligns with our expectations.
\Ours still mitigate the distribution shift better than DIRL methods.

\section{Real Data Experiments}
\label{sec:real-data-exp}
For real-world graphs, we will compare \Ours with domain adaptation algorithms designed for neural networks and graph neural networks in both supervised and semi-supervised learning settings. 

\xhdr{Baselines.} In addition to the domain adaptation algorithms used in previous sections, we consider the following methods for comprehensive study under distribution shift: (1) MMD~\cite{mmd2015} and (2) DANN~\cite{ganin2016domain}. 
For graph-specific methods, we choose three representative methods: (1) UDAGCN~\cite{udagcn2020} couples domain adversarial learning with graph attention mechanism; (2) SRGNN-IW~\cite{zhu2021shift} proposes to use instance weighting technique on GNN output embeddings; (3) Graph-EERM~\cite{eerm2022} proposes to augment training graph for invariance principles in risk minimization. As for our own ablations, we report the performance of DIRL version of our model ($\beta=0$ in in Eq.~\eqref{eq:pairwise-distance}) besides two variants of our methods \Ours and \Ours++.  All models are trained a single Nvidia A6000 GPU. Configurations of different algorithms on each dataset can be found in Appendix \ssym B.2

\begin{table}[ht]
\caption{Semi-supervised classification on three different citation networks with OOD training samples. Results from the original paper~\cite{zhu2021shift} are marked $^{\dag}$. We mark the \textbf{best} and the \underline{second} best results.
}
\centering

\scalebox{0.8}{
\begin{tabular}{ll|c c c c c c c c c}
\toprule
\multicolumn{2}{c }{\multirow{2}{*}{Method}}                             & \multicolumn{3}{c }{\texttt{Cora}}            & \multicolumn{3}{c }{\texttt{Citeseer}} & \multicolumn{3}{c}{\texttt{PubMed}}                                                                                                                      \\
\multicolumn{2}{c}{}                                                    & Micro-F1 & Macro-F1 & $\Delta$Acc  & Micro-F1 & Macro-F1 & $\Delta$Acc & Micro-F1 & Macro-F1 & $\Delta$Acc  \\ \midrule
\multicolumn{2}{l}{IID training}                   &  80.8 \scriptsize{$\pm$ 1.5}       &  80.1 \scriptsize{$\pm$ 1.3}    &  0               &  70.2 \scriptsize{$\pm$ 1.9}               &  66.8 \scriptsize{$\pm$ 1.7}       &  0     &   79.7 \scriptsize{$\pm$ 1.4}  & 78.8 \scriptsize{$\pm$ 1.4}    &       0                 \\
\midrule
\multicolumn{2}{l}{OOD training}                     & 71.3 \scriptsize{$\pm$ 4.1}       &  69.2 \scriptsize{$\pm$ 3.4}   &  9.5               &  63.4 \scriptsize{$\pm$ 1.8}   & 61.2 \scriptsize{$\pm$ 1.6}        &  6.9    & 63.4 \scriptsize{$\pm$ 4.2} & 58.7 \scriptsize{$\pm$ 7.0} & 16.4 \\
\multicolumn{2}{l}{MMD} &  71.5 \scriptsize{$\pm$ 4.9}       &  69.5 \scriptsize{$\pm$ 4.6}   &  9.3         &  64.4 \scriptsize{$\pm$ 1.2}               &  62.0 \scriptsize{$\pm$ 1.1}     &  5.9     &  66.3 \scriptsize{$\pm$ 4.2}     & 63.5 \scriptsize{$\pm$ 5.9}   &   13.4   \\
\multicolumn{2}{l}{CMD$^{\dag}$}               & \underline{72.1 \scriptsize{$\pm$ 4.4}}         &   69.8 \scriptsize{$\pm$ 3.7}   &  8.7   &  63.9 \scriptsize{$\pm$ 0.7}    & 61.8 \scriptsize{$\pm$ 0.6}        &   6.4       & 69.4\scriptsize{$\pm$ 3.4}         &   67.6 \scriptsize{$\pm$ 4.0} & 10.4               \\

\multicolumn{2}{l}{DANN}    &  71.5 \scriptsize{$\pm$ 5.0}       &  69.5 \scriptsize{$\pm$ 4.6} &     9.3     & 64.7 \scriptsize{$\pm$ 1.2} & 62.3 \scriptsize{$\pm$ 1.1} & 5.6  & 64.5 \scriptsize{$\pm$ 4.9} & 60.6 \scriptsize{$\pm$ 7.8} & 15.2   \\

\multicolumn{2}{l}{CDAN}                   &  71.5 \scriptsize{$\pm$ 5.1}   &   69.5 \scriptsize{$\pm$ 4.7}  &   9.3   & 64.6 \scriptsize{$\pm$ 1.3}  & 62.2 \scriptsize{$\pm$ 1.2}  &  5.6  & 64.1 \scriptsize{$\pm$ 5.0}  & 59.9 \scriptsize{$\pm$ 7.9} &   15.6             \\
  \midrule
  \multicolumn{2}{l}{UDAGCN}                   &  36.2 \scriptsize{$\pm$ 4.5}     & 35.4 \scriptsize{$\pm$ 4.3} &     44.6  & 33.8 \scriptsize{$\pm$ 5.1} & 31.5 \scriptsize{$\pm$ 7.7} &  36.4  & 40.6 \scriptsize{$\pm$ 6.8} & 34.9 \scriptsize{$\pm$ 6.8} &  39.1   \\
\multicolumn{2}{l}{EERM}                   &  68.3 \scriptsize{$\pm$ 4.3}       &  66.2 \scriptsize{$\pm$ 3.9}   &  12.5               &  62.3 \scriptsize{$\pm$ 1.0}               &  59.5 \scriptsize{$\pm$ 1.0}       &  7.9     &  61.6 \scriptsize{$\pm$ 4.8}  & 56.8 \scriptsize{$\pm$ 7.7}   &       18.1                 \\
\multicolumn{2}{l}{SRGNN-IW$^{\dag}$}                 & 72.0 \scriptsize{$\pm$ 3.2}         &   69.5 \scriptsize{$\pm$ 3.7}   &  8.8  &  \textbf{66.1 \scriptsize{$\pm$ 0.9}}   & \underline{63.4 \scriptsize{$\pm$ 0.9}}       &   4.2 &  66.4 \scriptsize{$\pm$ 4.0}               &  64.0  \scriptsize{$\pm$ 5.5} &  13.4 \\

  \midrule

\multicolumn{2}{l}{\Ours-DIRL}                   &  71.7 \scriptsize{$\pm$ 4.7}      & 69.7 \scriptsize{$\pm$ 4.3} &    9.1 & 64.6 \scriptsize{$\pm$ 1.1} & 62.2 \scriptsize{$\pm$ 1.0} & 5.6 & 68.3 \scriptsize{$\pm$ 3.9} & 66.5 \scriptsize{$\pm$ 4.7} & 11.4       \\
\multicolumn{2}{l}{\Ours}                   &  71.7 \scriptsize{$\pm$ 4.7}      & \underline{70.2 \scriptsize{$\pm$ 2.7}} &   9.1  & 65.3 \scriptsize{$\pm$ 0.8} & {63.3 \scriptsize{$\pm$ 0.8}} &    4.9 & \underline{71.5 \scriptsize{$\pm$ 2.9}} & \underline{70.4 \scriptsize{$\pm$ 3.1}} & 8.2        \\
\multicolumn{2}{l}{\Ours++}                 &  \textbf{72.6 \scriptsize{$\pm$ 3.1}}        &   \textbf{70.7\scriptsize{$\pm$ 3.0}}               & 8.2 &    \underline{65.6 \scriptsize{$\pm$ 0.9}}   &  \textbf{63.5 \scriptsize{$\pm$ 0.9}}       &   4.6    &  \textbf{73.0 \scriptsize{$\pm$ 2.5}}                &  \textbf{71.9 \scriptsize{$\pm$ 2.5}}  &  6.7 \\
\bottomrule
\end{tabular}}
\label{tab:semi-supervised}
\vspace{-1em}
\end{table}

\subsection{Semi-supervised Node Classification} \vspace{-3pt}
GNNs are widely recognized for their effectiveness in node classification tasks, particularly when dealing with a limited amount of labeled data. In semi-supervised classification, source data is a small number of training nodes and target data are all of the remaining nodes in the same graph. 
Recently, SRGNN~\cite{zhu2021shift} found biased training data in semi-supervised learning can cause dramatic accuracy loss; they provide the algorithm to generate biased training nodes (refered to as OOD training in Table \ref{tab:semi-supervised}) on three semi-supervised learning benchmarks: Cora, Citeseer and PubMed~\cite{sen2008collective}. 
We choose the best-performing GNN architecture from their paper - APPNP~\cite{klicpera2018predict} and report the Micro-F1, and Macro-F1 for each method and the accuracy loss compared with IID training data. We are able to reproduce the performance gap between IID and OOD training data ($\Delta$ in in Table \ref{tab:semi-supervised}). 
We begin by noting that most of the general domain adaptation algorithms such as CMD, MMD, and DANN can help improve the performance because conditional shift is small in this setting.
Among these algorithms, we find that directly optimizing discrepancy metrics seems to be more effective and robust (smaller average loss and deviation over 100 runs) than adversarial methods (CDAN and DANN) which often require more tuning. Across the three datasets, \Ours++ consistently achieves top-2 performance, while \Ours (\ie, only optimizing conditional shift) generally ranks second best. In addition, \Ours-DIRL demonstrates similar performance to DIRL methods such as CMD and MMD. These observations suggest that the primary improvements stem from minimizing the estimated conditional shift $\widehat{\W}_1(\mu_\S^f, \mu_\T^f)$.

\begin{table}[h]
\caption{Domain adaptation on node and graph classification. We mark the \textbf{best} and the \underline{second} best accuracy.}
\centering
\scalebox{0.8}{
\setlength{\tabcolsep}{4pt}{
\begin{tabular}{l ccc|c | ccc|c}\\\toprule  
{\multirow{2}{*}{Method}}                             & \multicolumn{4}{c }{Node Classification (Micro-F1)}            & \multicolumn{4}{c }{Graph Classification (AUC)} \\
 & \small{ACM-DBLP}$_\text{small}$ & \small{ACM}$_\text{time}$ & \small{ACM-DBLP}$_\text{large}$ & Avg. $\Delta$ & BACE & BBBP & Clintox & Avg. $\Delta$ \\ \midrule
Base model   &  68.1 \scriptsize{$\pm$ 2.1}   &  78.8 \scriptsize{$\pm$ 1.0} &    81.1 \scriptsize{$\pm$ 0.2} & &  64.8 \scriptsize{$\pm$ 2.8} & 71.0 \scriptsize{$\pm$ 8.7} & 52.8 \scriptsize{$\pm$ 3.3} \\  
CMD$^{\dag}$  &   \underline{75.5 \scriptsize{$\pm$ 4.4}}             &   79.4 \scriptsize{$\pm$ 0.7}    &   75.2 \scriptsize{$\pm$ 0.8} & +0.97 & 60.4 \scriptsize{$\pm$ 1.4} & 72.0 \scriptsize{$\pm$ 1.8} & 55.0 \scriptsize{$\pm$ 5.0} & -0.40 \\ 
DANN  & 70.1 \scriptsize{$\pm$ 1.8}           & 79.6 \scriptsize{$\pm$ 0.4}       & 81.6 \scriptsize{$\pm$ 0.4} & +1.10  & 67.4 \scriptsize{$\pm$ 2.9} & \underline{74.0 \scriptsize{$\pm$ 2.3}} & \underline{61.6 \scriptsize{$\pm$ 3.6}} & \underline{+4.80} \\
CDAN                   &  75.3 \scriptsize{$\pm$ 4.3}  &         79.3 \scriptsize{$\pm$ 1.3} & {82.1 \scriptsize{$\pm$ 0.3}} & \underline{+2.90} &     \textbf{69.1 \scriptsize{$\pm$ 1.8}} & {73.5 \scriptsize{$\pm$ 1.7}} & {57.5 \scriptsize{$\pm$ 2.4}} & +3.83  \\
\midrule
UDAGCN      & 66.4 \scriptsize{$\pm$ 5.1}  &         79.3 \scriptsize{$\pm$ 0.5}  &         78.3 \scriptsize{$\pm$ 2.6} & 
 -1.33 & \underline{67.9 \scriptsize{$\pm$ 1.4}} & 73.3 \scriptsize{$\pm$ 2.1} & 60.7 \scriptsize{$\pm$ 4.8} & + 4.43 \\
EERM            & 64.9 \scriptsize{$\pm$ 3.5} &       77.3 \scriptsize{$\pm$ 0.4} &         81.0 \scriptsize{$\pm$ 0.4}  &  -1.60 & N/A & N/A & N/A & N/A \\
SRGNN-IW                 & 69.2 \scriptsize{$\pm$ 1.6}  &         79.5 \scriptsize{$\pm$ 1.1} &        81.4 \scriptsize{$\pm$ 0.4} & 0.70 & 65.2 \scriptsize{$\pm$ 3.3} & 71.7 \scriptsize{$\pm$ 2.8}& 57.3 \scriptsize{$\pm$ 3.6} & +1.87 \\ 
\midrule
\Ours-DIRL                  & 71.6 \scriptsize{$\pm$ 2.3} &         \underline{80.2 \scriptsize{$\pm$ 0.4}} &          \underline{82.3 \scriptsize{$\pm$ 0.4}} & +2.03 & 65.4 \scriptsize{$\pm$ 2.4} & 69.3 \scriptsize{$\pm$ 4.0} & 57.9 \scriptsize{$\pm$ 3.6} & +1.33 \\
\Ours             & 74.0 \scriptsize{$\pm$ 4.7}  &         {80.1 \scriptsize{$\pm$ 0.5}} &        {82.1 \scriptsize{$\pm$ 0.3}} & +2.73 & {64.7 \scriptsize{$\pm$ 2.0}} & 70.0 \scriptsize{$\pm$ 4.2} & 57.2 \scriptsize{$\pm$ 2.1} & +1.10    \\

\Ours++                &  \textbf{78.5 \scriptsize{$\pm$ 4.0}} &         \textbf{80.3 \scriptsize{$\pm$ 0.8}} &      \textbf{82.5 \scriptsize{$\pm$ 0.3}} & \textbf{+4.43} &  67.8 \scriptsize{$\pm$ 2.5} & \textbf{74.4 \scriptsize{$\pm$ 3.0}} & \textbf{61.7 \scriptsize{$\pm$ 2.4}} & \textbf{+4.83} \\
\bottomrule
\end{tabular}
}
}
\vspace{-1em}
\label{tab:supervised-exp}
\end{table}
\subsection{Supervised Node and Graph Classification}
In a fully-supervised setting, transfer learning is commonly employed to transfer knowledge across different domains for graph-structured data.  This involves training a model on source graphs and inferring on target graphs.
We conduct domain adaptation experiments on citation networks \cite{tang2008arnetminer} and molecular graphs \cite{hu2020open} for two tasks. The first task involves node classification by introducing domain shift between ACM and DBLP graphs, as well as time shift within the ACM graphs. The second task focuses on graph classification with scaffold shift, where the training and testing molecular graphs have different scaffold patterns.
For node classification and graph classification, we adopt a 2-layer GCN~\cite{kipf2016semi} and a 5-layer GraphSAGE~\cite{hamilton2017inductive}, respectively, following established practices. Specifically, for graph classification, we employ mean pooling to obtain the graph representations.

In Table~\ref{tab:supervised-exp}, we make several key observations: (1) different algorithms exhibit varying performance under different settings, primarily due to the presence of various types of distribution shift; (2) on node classification, \Ours and its variants usually outperforms the other baselines with a clear margin. This can be attributed to the fact that our approach has been theoretically designed to excel in node classification scenarios; (3) Domain adaptation algorithms, such as DANN, that are originally designed for neural networks exhibit better performance in graph classification tasks, because the graph classification task shares closer similarities with the image domain. Nevertheless, it is noteworthy that \Ours++ consistently achieved top-3 rankings across all tasks and highest average improvement (\ie Avg. $\Delta$), indicating our potential usage on graph property predictions. For further details on the dataset and complementary experiments, please refer to Appendix \ssym C.

\section{Conclusion}
\label{sec:con}
In this work we establish the first theoretical connection between the inductive bias of GNNs and distribution shift by quantifying conditional shift.
Our novel theoretical results show that conditional shift is often exacerbated by GNNs, explaining the 
limited performance of popular DIRL methods on graph data.
To remedy this shift in the latent space, we present a graph domain adaptation framework based on our theoretical results. 
Using a number different experiments on both synthetic and real data , we demonstrate that our method \Ours results in a robust improvement on different kinds of domain shifts.
As for future work, we have two notable directions to explore: (1) extend our analysis to other types of graph neural networks (2) develop more advanced GNNs following our theoretical results for graph domain adaptation. 

\unhidefromtoc

\clearpage
\bibliography{ref}
\bibliographystyle{plain}
\newpage

\appendix

\tableofcontents
\newpage

\section{Theory details}
\subsection{Proof of Theorem 3.1}
\label{proof_1}
In Definition 3.3, we made several simplifications on original CSBM model to investigate its OOD generalization \wrt structure and feature distribution shifts. The original $\text{CSBM}(\mu,\nu,p,q)$ is defined to have two different class means $\mu$ and $\nu$.
Given training and testing graphs as $\G_\S \sim \text{CSBM}(\mu,\nu,p,q)$ and $\G_\T \sim \text{CSBM}(\mu',\nu',p',q')$, we let $\nu = -\mu$ in CSBM by making $\vec{0}$ the middle point of original feature mean of two classes. Without loss of generality, we let two graphs have same amount of nodes $n=n'$ and edge density $D=D'$. 
Here we restate the pseudo conditional shifts $\Delta_{\y|\x}$ on the hypothesis function used in Theorem 3.1. In this context, the function $d$ is defined as an indicator function, which serves as a realization of Definition 3.2.

\begin{equation}
\Delta_{\y|\x} = \mathbb{E}_{\x \sim \P_t(\x)}\left(\mathbb{I}\left[\arg\max_y\P_s(\y|\x)\neq\arg\max_y\P_t(\y|\x) \right]\right),
    \label{eq:supp_conditional_shift}
\end{equation}


\begin{theorem}[{Conditional Shift in GNNs}]
Let the source graph $\mathcal{G}_\S$ = CSBM($\mu$, $p$, $q$), and a target graph $\mathcal{G}_\T$ = CSBM($\mu'$, $p'$, $q'$), where $D$ and $D'$ represent their average degrees respectively. Additionally, let $\Phi(\cdot)$ denote the cumulative distribution function (CDF) of a multivariate Gaussian distribution defined by distance. 
Then the introduced distribution shift between $\mathcal{G}_\S$ and $\mathcal{G}_\T$ can be quantified via the estimated conditional shift of  $\x$ and $\h$ as:
\begin{equation} \label{eq:supp_gcn_shift}
\Delta_{\y|\x} = \frac{\Phi\left((1+\delta)\| \mu \|) - \Phi((1-\delta)\| \mu\| \right) }{2}, \Delta_{\y|\h} = \frac{\Phi(\|\mu_{h,-1}' \|) - \Phi(\| \mu_{h,1}'\| ) }{2}, 
\end{equation}

where $\mu_{h,1}' = \sqrt{D'}\frac{p'-q'}{p'+q'}\mu - \sqrt{D'} \delta\mu $ and $\mu_{h,-1}' = \sqrt{D'}\frac{q'-p'}{p'+q'}\mu - \sqrt{D'} \delta \mu$.
\end{theorem}

\begin{prop}
Through training with hinge loss, the linear classifier $f$ on original feature $\x$ and GNN latent space $\h$ have the same optimal hyperplane $\mathcal{P}=\{\x|\w^Tx+b=0\}$ characterized by $\f(\w^*, b^*)$, $\w^* = \mu$ and $b^* =0$.
\end{prop}

\begin{proof}

On a CSBM graph $\G(\mu,p,q)$, 
the data distribution on feature $\x$ is,
\begin{equation}
\begin{split}
    {x}_i & \sim \N\left( \mu,  \mathbf{I}\right), \ y_i=1, \\
    {x}_i & \sim \N\left(-\mu,  \mathbf{I}\right), \ y_i=-1
\end{split}
\end{equation}

Since $\x$ is a standard Gaussian, the output of the $\arg\max$ operator is identical to the optimal $\f^{*}(x)$.
Furthermore, the distributions on the source and target share the same support. Thus, the indicator function in Equation~\ref{eq:supp_conditional_shift} can be simplified as the expected difference in predictions between the optimal source classifier $\f$ and the optimal target classifier $\f^\prime$ on the target data, that is,
\begin{equation}
\Delta_{\y|\x}  = \mathbb{E}_{\x \sim \P_t(\x)} \left(\mathbb{I}\left[\f(x)\neq \f^\prime(x) \right]\right),
\end{equation}

We first discuss the conditional shift on the feature $\x$ of the target graphs. Since we assume that the distribution shift on the feature $\mu$ is controlled by $\delta$, the centers of the two classes on the target graphs are located at $-(1+\delta)\mu$ and $(1-\delta)\mu$:
\begin{equation}
\begin{split}
    {x}_i^\prime & \sim \N\left( (1-\delta)\mu,  \mathbf{I}\right), \ y_i=1, \\
    {x}_i^\prime & \sim \N\left(-(1+\delta)\mu,  \mathbf{I}\right), \ y_i=-1,
\end{split}
\label{eq:supp_pdf_target}
\end{equation}
The optimal classifier is $\f(\mu, 0)$ on source graph and $\f^\prime(\mu, \delta \mu)$ on target CSBM graph. We further partition the computation of $\Delta_{\y|\x}$ on two classes, that is $\Delta_{\y=1|\x}+\Delta_{\y=-1|\x}$. When $y_i^\prime=1$, the different predictions (\ie $f(x_i)\neq f'(x_i)$) are those samples between $0$ and $-\delta \mu$ in 1-dimension case. Considering the probability density function $\P_t(\x)$ in Equation~\ref{eq:supp_pdf_target}, $\Delta_{\y=1|\x}$ is calculated as,
\begin{equation}
    \Delta_{\y=1|\x} = \frac{1}{\sqrt{2 \pi}} \int_{-\delta \mu}^{0} \exp(\{-\frac{(t-(1-\delta)\mu)^2}{2}\}) \,dt
\end{equation}

The CDF of the standard Gaussian distribution is denoted by the $\Phi$ function.
\begin{equation}
    \Phi(x) = \frac{1}{\sqrt{2 \pi}} \int_{-\infty}^{x} \exp(\{-\frac{t^2}{2}\}) \,dt
\end{equation}
In standard multivariate (d > 1) Gaussian distribution, we define the CDF as a monotonic function regarding the distance to the Gaussian mean $\Phi(\|\cdot\|)$.

To represent the conditional shift use $\Phi$, we flip the axis $x=-x$ and translate the distribution into a standard Gaussian $\N(0, \mathbf{I})$ by moving $1- \delta \mu$ as described in Figure~\ref{fig:cdf_proof}.
\begin{equation}
    \Delta_{\y=1|\x} = \Phi(\|\mu\|) - \Phi(\|\mu-\delta\mu\|),
\end{equation}
\begin{figure}[h]
    \centering
    \includegraphics[width=0.95\linewidth]{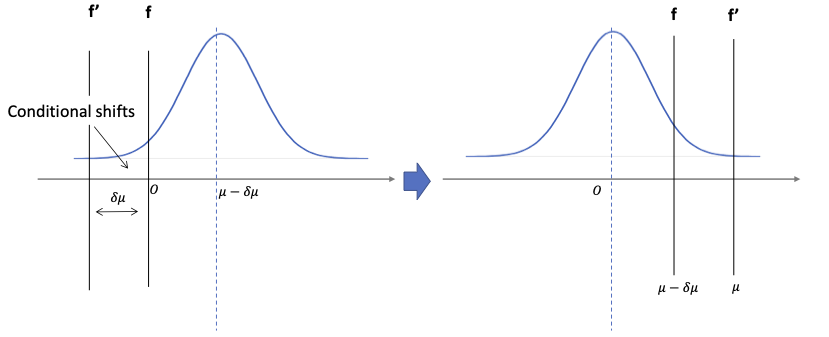}
    \caption{Representing $\Delta_{\y=1|\x}$ with $\Phi$}
    \label{fig:cdf_proof}
\end{figure}

Similarly, for class $y_i=-1$, we can have,
\begin{equation}
    \Delta_{\y=-1|\x} = \Phi(\|\mu+\delta\mu\|) - \Phi(\|\mu\|),
\end{equation}
Since we have the same amount of nodes for each class, we get the value of conditional shifts on original features and weighted average of two classes,
\begin{equation}
    \Delta_{\y|\x} = \frac{\Delta_{\y=-1|\x}+ \Delta_{\y=1|\x}}{2} = \frac{\Phi(\|\mu+\delta\mu\|) - \Phi(\|\mu-\delta\mu\|)}{2} =\frac{\Phi\left((1+\delta)\| \mu \|) - \Phi((1-\delta)\| \mu\| \right) }{2},
\end{equation}

Now, we are ready to discuss the conditional shift on the GCN transformed features $\h$.
Now, we are ready to discuss the conditional shift on GCN transformed features $\h$.
The feature of a node in a CSBM graph that has been transformed using GCN is obtained as a weighted mean of $D$ (average degree) distinct Gaussian random variables. Among these variables, $\frac{p}{p+q}$ constitute the intra-class variables, while $\frac{q}{p+q}$ make up the inter-class variables. As a result, the data distribution on $\h$ is as follows:
\begin{equation}
\begin{split}
    {h}_i &\sim \N\left( \frac{p-q}{p+q}\mu, \frac{1}{\sqrt{D}} \mathbf{I}\right), \ y_i=1, \\
    {h}_i &\sim \N\left( \frac{q-p}{p+q}\mu, \frac{1}{\sqrt{D}} \mathbf{I}\right), \ y_i=-1
\end{split}
\end{equation}


We rescale the Gaussian distribution output by graph convolution to standard Gaussian distribution,
$$
h_i  \sim \N \left( \sqrt{D} \cdot \frac{p-q}{p+q}\mu, \mathbf{I} \right), \text { for } y_i=1.
$$
Consequently, on target graph,
\begin{equation}
\begin{split}
    {h}_i^\prime & \sim \N\left( \sqrt{D'} \cdot \frac{p'(1-\delta)\mu-q'(1+\delta)\mu'}{p'+q'},  \mathbf{I}\right), \ y_i=1, \\
    {h}_i^\prime & \sim \N\left( \sqrt{D'} \cdot \frac{q'(1-\delta)\mu-p'(1+\delta)\mu'}{p'+q'},  \mathbf{I}\right), \ y_i=-1
\end{split}
\end{equation}
Let $\mu_{h,1}' = \sqrt{D'}\frac{p'-q'}{p'+q'}\mu - \sqrt{D'} \delta\mu $ and $\mu_{h,-1}' = \sqrt{D'}\frac{q'-p'}{p'+q'}\mu - \sqrt{D'} \delta \mu$, we are ready to finish the proof by calculation conditional shifts on target data.
\begin{equation}
    \Delta_{\y|\h} = \frac{1}{2} \left( \Phi(\frac{\mu_{h,-1}'^\intercal\mu_h}{\| \mu \|}) - \Phi(\frac{\mu_{h,1}'^\intercal\mu_h}{\| \mu_h \|}\right) = \frac{\Phi(\|\mu_{h,-1}' \|) - \Phi(\| \mu_{h,1}'\| ) }{2}.
\end{equation}

\end{proof}

Now let's discuss the relative conditional shift on $\x$ and $\h$ when structure or feature deviates from training, respectively.
\begin{corollary} [GNNs exacerbate Conditional Shift]
Assuming only homophily ratio changes $p/q \neq p'/q'$, the conditional shift is always exacerbated by the 1-layer GCN since $\Delta_{\y|\x}=0$.
When there is only a feature shift $\delta \mu$, the shift will be amplified by the GCN as $\sqrt{D}\delta\mu$, potentially leading to larger conditional shifts.
\end{corollary}
\begin{proof}
When graph structure $(p^\prime, q^\prime)$ changes on target graph while $\mu$ remains the same (\ie $\delta=0$), $ \Delta_{\y|\h} \geq  \Delta_{\y|\x} =0 $

When there is a distribution shift in the feature mean of the class ($\delta > 0$), we define $\mu' = \sqrt{D'} \frac{p' - q'}{p' + q'}\mu$ and obtain the following expression:
\begin{equation}
    \Delta_{\x|\h} =\frac{\Phi(\|\mu+\delta\mu\|) - \Phi(\|\mu-\delta\mu\|)}{2},
    \Delta_{\y|\h} = \frac{\Phi(\|\mu'+\sqrt{D'}\delta\mu \|) - \Phi(\|\mu'-\sqrt{D'}\delta\mu\| ) }{2}.
\end{equation}
Although obtaining a closed-form solution for when GCNs exacerbate conditional shift, i.e., $\Delta_{\x|\h} < \Delta_{\y|\h}$, is complicated, we can analyze the effect of varying $\delta$ on $\x$ and $\h$. We observe that both $\Delta_{\x|\h}$ and $\Delta_{\y|\h}$ are monotonically increasing functions of $\delta$. In the latent space $\h$ \wrt $\mu'$, the magnitude of feature shift is amplified by $\sqrt{D'}$.

\end{proof}

\begin{corollary} [Relation between conditional shift and generalization]
Conditional shift upper bounds the performance gap between source and target, \ie $\Delta > |\varepsilon_\T-\varepsilon_\S|$.
The expected target error $\varepsilon_{\T}$ for linear classifiers $f$ and GNNs $f\circ g$ are,
\begin{equation}
    \varepsilon_{\T}(f)= 1 -  \frac{\Phi((1+\delta)\| \mu \| ) + \Phi((1-\delta)\| \mu \|)}{2}, \varepsilon_{\T}(f\circ g)=1-\frac{\Phi(\|\mu_{h,-1}' \|) + \Phi(\| \mu_{h,1}'\| ) }{2}.
\end{equation}
\end{corollary}
\begin{proof}
We begin by computing the expected target error $\varepsilon_{\T}$ on $\x$, denoted as $\varepsilon_{\T}(f)$. Unlike the calculation of conditional shift, the expected error is evaluated on the target graph $\mathcal{G}_\T$ and can be expressed as follows:
\begin{equation}
    \varepsilon_{\T}(f)  = \mathbb{E}_{\x \sim \P_t(\x)} \left(\mathbb{I}\left[\f(x)\neq y \right]\right),
\end{equation}
We recall that the feature means of the two classes are $(-(1+\delta)\mu, (1-\delta)\mu)$.
For class $1$, the optimal $f$ fails to classify $x_i$ correctly if $w^T x_i+b<0$, with a distance of $|(1-\delta)\mu|$ or more from $(1-\delta)\mu$. The probability of such instances can be calculated as $1-\Phi(|(1-\delta)\mu|)$. Combining this with class $-1$, we obtain the following result:
\begin{equation}
    \varepsilon_{\T}(f)=\underbrace{\frac{1-\Phi(\|\mu+\delta\mu\|)}{2}}_\text{error of class -1} + \underbrace{\frac{1-\Phi(\|\mu-\delta\mu\|)}{2}}_\text{error of class 1}
\end{equation}

Similarly, on a source graph, the expected error is $\varepsilon_{\S}(f) = 1 - \Phi(\mu)$ and $\Phi(\cdot)$ is a monotonically increasing function. Therefore, we have $|\varepsilon_{\S} - \varepsilon_{\T}| = \Phi(|\mu|) - \frac{\Phi(|\mu+\delta\mu|) + \Phi(|\mu-\delta\mu|)}{2}$. Furthermore, we can calculate $\Delta_{\y|\x} - |\varepsilon_{\S} - \varepsilon_{\T}|$ as follows:
\begin{equation}
    \Delta_{\y|\x} - |\varepsilon_{\S}-\varepsilon_{\T}| = \Phi(\|\mu+\delta\mu\|) - \Phi(\|\mu\|) > 0
\end{equation}

Regarding graph convolution networks, the class centroids after GCN are $\mu_{h,1}'$ and $\mu_{h,-1}'$ as calculated in Theorem 3.1. The expected error of a linear classifier $f$ on the output of GCN $g$ is obtained as follows:
\begin{equation}
    \varepsilon_{\T}(f\circ g)=1-\frac{\Phi(\|\mu_{h,-1}' \|) + \Phi(\| \mu_{h,1}'\| ) }{2}
\end{equation}

We re-use the definition of $\mu'$ from the proof of Corollary 3.1.1. We can now complet the proof:
\begin{equation}
    \Delta_{\h|\x} - |\varepsilon_{\S}-\varepsilon_{\T}|=\Phi(\|\mu'+\sqrt{D'}\delta\mu \| - \Phi(\|\mu'\|) > 0
\end{equation}
\end{proof}

\subsection{Proof of Theorem 4.1}

\begin{theorem}
Suppose $\mathcal{F}$ is the hypothesis space of GNNs, $\forall f \in \mathcal{F}$,
\begin{equation}
\varepsilon_\T(f) \leq \varepsilon_\S(f) + \W_1(\mu_\S^{f}, \mu_\T^{f}) + \lambda^* + K_\mathcal{L} K_g \phi(c),
\end{equation}
where $\lambda^*$ is the joint optimal error, $K_\mathcal{L}$ is the Lipschitz constant loss function of loss function $\mathcal{L}$, $K_g$ is the Lipschitz constant of GNN $g$ and $\phi(c)$ is the probabilistic lipschitzness~\cite{ben2014domain}.
\end{theorem}
\begin{proof}
Following the approach in~\cite{courty2017joint}, we introduce $f^*$ as the optimal labeling function in the hypothesis space $\mathcal{F}$, giving us:
\begin{align}
\label{eq:thm4_1}
    \varepsilon_\T(f) &= \mathbb{E}_{(\x,\y) \sim \P_t}\mathcal{L}(\y, f(\x)) \nonumber \\
    &\leq \mathbb{E}_{(\x,\y) \sim \P_t}\mathcal{L}(\y, f^*(\x) + \mathcal{L}(f(\x), f^*(\x)) \nonumber \\
    &= \mathbb{E}_{(\x,\y) \sim \P_t}\mathcal{L}(f(\x), f^*(\x)) + \varepsilon_\T(f^*) \nonumber \\
    &=  \mathbb{E}_{(\x,f(\x)) \sim \P_t^f}\mathcal{L}(f(\x), f^*(\x)) + \varepsilon_\T(f^*) \nonumber \\
    &= \varepsilon_\T(f^*)-\varepsilon_\S(f^*)+\varepsilon_\S(f^*)+\varepsilon_\T(f^*) \nonumber \\
    &\leq |\varepsilon_\T^f(f^*) - \varepsilon_\S(f^*)|+\underbrace{\varepsilon_\S(f^*)+\varepsilon_\T(f^*)}_{\lambda^*}
\end{align}

Now we introduce the definition of $K_g$ and $\phi(c)$ in the theorem.
The Lipschitz constant of GNNs has garnered considerable attention in recent studies~\cite{chuang2022tree}. In our analysis, we view the data distribution as rooted subtrees~\cite{egi2020} centered around node $i$, denoted as $x_i = T_i$, where $T_i$ are sampled from graph $G$. We define the Lipschitz constant $K_g$ of GNNs as follows:
\begin{equation}
    |f(T_i)-f(T_j)| \leq K_g |l(T_i)-l(T_j)| \leq K_g
\end{equation}
where $l: T_i \rightarrow [0,1]^d$ is a bounded function maps node features in the rooted subtree to real values, \eg mean aggregation and normalization in GraphSAGE~\cite{hamilton2017inductive}.

\begin{definition}[Probabilistic Transfer Lipschitzness~\cite{courty2017joint}] Let $\phi:\mathbb{R}\rightarrow[0,1]$, a labeling function $f:X\rightarrow\mathbb{R}$ and a joint distribution $\Gamma$ over $\mu_\S$ and $\mu_\T$, the $\phi$-transfer lipschitzness represents for all $c$:
\begin{equation}
    \P_{(\x_s,\x_t)}[|f(\x_s)-f(\x_t)|>c d(\x_s,\x_t)] \leq \phi(c)
\end{equation}
\end{definition}

Let $\mu_\S^f = \P_s(\x,\y)$ and $\mu_\T^f = \P_t^f(\x,f(\x))$ denote the source data distribution and the estimated target data distribution, respectively.
$\varepsilon_\T^f(f^*)$  can be interpreted as the discrepancy in predictions between $f$ and $f^*$. Given $\Gamma^*$ is the optimal transportation plan of \Ours, we have:
\begin{align}
&|\varepsilon_\T^f(f^*) - \varepsilon_\S(f^*)| \nonumber \\
&= \Big\rvert \int \mathcal{L}(y,f^*(\x))\d(\P_t^f-\P_s)\Big\rvert \nonumber \\
&= \Big\rvert \int \left(\mathcal{L}(f(\x_t),f^*(\x_t)) - \mathcal{L}(\y_s,f^*(\x_s)) \right)\d\Gamma^*((\x_s,\y_s),(\x_t,f(\x_t)))\Big\rvert \nonumber \\
&\leq \int \Big\rvert\left(\mathcal{L}(f(\x_t),f^*(\x_t))- \mathcal{L}(\y_s,f^*(\x_s)) \right)\Big\rvert\d\Gamma^*((\x_s,\y_s),(\x_t,f(\x_t))) \nonumber \\
&\leq \int \Big\rvert\mathcal{L}(f(\x_t),f^*(\x_t)) - \mathcal{L}(f(\x_t),f^*(\x_s))\Big\rvert + \Big\rvert \mathcal{L}(f(\x_t),f^*(\x_s)) - \mathcal{L}(\y_s,f^*(\x_s)) \Big\rvert\d\Gamma^*((\x_s,\y_s),(\x_t,f(\x_t))) \nonumber \\
&\leq \int K_\mathcal{L} \Big\rvert f^*(\x_s) - f^*(\x_t) \Big\rvert + \mathcal{L}(f(\x_t),\y_s) \d\Gamma^*((\x_s,\y_s),(\x_t,f(\x_t)))  \\
&\leq \int c*K_\mathcal{L}d(\x_s,\x_t) +\mathcal{L}(y_s, f(\x_t)) \d\Gamma^*((\x_s,\y_s),(\x_t,f(\x_t))) + K_\mathcal{L} K_g\phi(c) \\
&\leq \W_1(\mu_\S^f, \mu_\T^f) + K_\mathcal{L} K_g\phi(c) 
\end{align}
Line (25) is a consequence of Lipschitz constant and triangle inequality on $\mathcal{L}$. Line (26) applies $\phi(c)$-transfer lipschitzness on $f^*(\x)$.
The last line (27) is achieved by setting $\alpha=c*K_\mathcal{L}$ in Eq. (11) of the main paper.
We complete the proof by combining Eq.~\eqref{eq:thm4_1} and Eq. (27).
\end{proof}

\subsection{Additional discussion on DIRL}
In Section 3.1 of the main paper, we introduced the covariate shift assumption on DIRL, which alternatively assumes a small conditional shift. However, even with this assumption, our synthetic experiment in Section 5 shows that the best DIRL method (\ie CMD) still yield unsatisfactory results. To further illustrate this from a theoretical perspective, we restate an existing study on the conditional shift in DIRL.

\begin{theorem} [Limits of learning invariant representations under conditional shift]~\cite{zhao2019learning}
Suppose markov chain $X\xrightarrow{\g} Z \xrightarrow{h}\hat{Y}$ and $d_\textnormal{JS}$ is the  Jensen-Shannon distance,
$$\varepsilon_\S(h\circ\g) + \varepsilon_\T(h\circ\g) \geq \frac{1}{2} \left( d_\textnormal{JS}(\D^Y_\S, \D^Y_\T) - d_\textnormal{JS}(\D^Z_\S, \D^Z_\T)^2 \right)$$
\end{theorem}
According to the above theorem, when $\P(Y|X)$ is different on source and target, minimizing source risk and  ${\mathcal{H}\Delta\mathcal{H}}$-divergence leads to a small JS distance $d_\textnormal{JS}(\D^Z_\S, \D^Z_\T)$. As a consequence, the marginal label shift $d_\textnormal{JS}(\D^Y_\S, \D^Y_\T)$ dominating the 
the lower bound of joint source and target risk. If conditional shift is large, DIRL cannot achieve accurate predictions on target. In Figure~\ref{fig:DANN_results}, we train a domain adversarial neural network~\cite{ganin2016domain} and project the node TSNE embeddings of source and target CSBM graphs. Two different colors indicate class labels, \textbf{O} dots are source data and \textbf{X} are target samples.
When the conditional shift is small and covariate shift assumption holds approximately, DANN can separate different classes well for both source and target domains (left). However, when there is large  conditional shift, the classification accuracy on target is low because it only minimizes discrepancy between representations, and classes end up intermixed.

\begin{figure}
  \begin{subfigure}{0.48\textwidth}
    \includegraphics[width=1\textwidth]{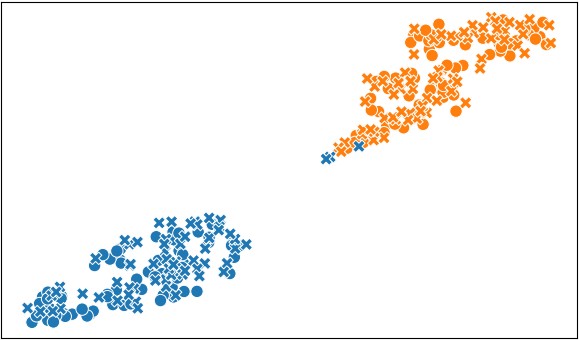}
    \caption{$\P_s(Y|X)\approx \P_t(Y|X)$}
  \end{subfigure}
\begin{subfigure}{0.48\textwidth}
    \includegraphics[width=1\textwidth]{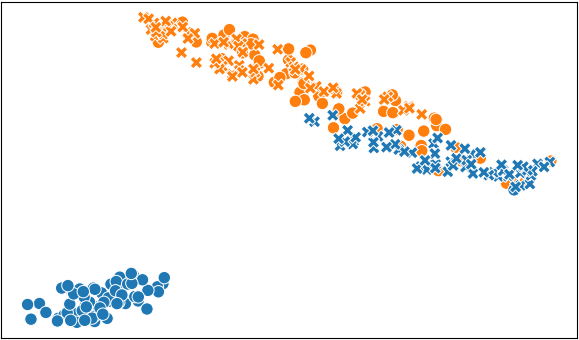}
  \caption{$\P_s(Y|X) \neq \P_t(Y|X)$}
\end{subfigure}
\caption{Performance of DIRL methods under small or large conditional shifts.}
\label{fig:DANN_results}
\end{figure}

\section{Model Details}
\label{supp:model}
\subsection{\Ours Algorithm}
\begin{algorithm}[H]
\textbf{Input:} Training graph $\G_\S$; testing graph $\G_\T$; \\ 
Graph Sampler \textbf{\texttt{SAMPLE}};\\
\textbf{Output:} GNNs $g$ and classifier $f$ with trained weights;\\
\For{\text{each batch of} $(\G^s_b, x^s_b, y^s_b)$ and ($\G^t_b, x^t_b)$ from \textbf{\texttt{SAMPLE}}}{
\text{fix $f,g$, compute $\d(\cdot)$ } $\leftarrow$ Eq. (10) of the main paper  \\
 \text{solve $\Gamma^*$ using an OT solver} \\
 fix $\Gamma^*$ and update the weights of $f,g$ $\leftarrow$ Eq. (9) of the main paper
 }
 \caption{Pseudo code for \Ours optimization}
 \label{alg}
\end{algorithm}
In the algorithm, we use node classification with a neighborhood sampler as an example. For graph classification, each sample $(\G_b, x_b, y_b)$ is a different graph sampled from source or target. 
\subsection{Implementations}
We implement our method and all other baselines using torch-geometric library. We list the graph neural network specifications used in our experiments,
\begin{enumerate}
    \item Synthetic node classification - model architecture: Graph Convolutional Networks~\cite{kipf2016semi}, hidden dimension: 16, activation: SiLU, number of layers: 2, dropout: 0.0
    \item Semi-supervised node classification - model architecture: APPNP~\cite{klicpera2018predict}, hidden dimension: 32, number of layers:2, dropout: 0.0, 
    \item Supervised node classification - model architecture: Graph Convolutional Networks~\cite{kipf2016semi}, hidden dimension: 128, activation: ReLU, number of layers: 2, dropout: 0.2
    \item Supervised graph classification - model architecture: GraphSAGE~\cite{hamilton2017inductive}, hidden dimension: 300, activation: ReLU, number of layers: 5, dropout: 0.5
\end{enumerate}
For supervised node classification, we utilized the RandomWalk GraphSAINT~\cite{zeng2019graphsaint} sampler with a batch size of 256, step size of 50, and walk length of 2. We indepdentently run experiments 10 times and report the mean and standard deviation in all table and figures.
All models are trained on a single Nvidia A6000 GPU.
The code for each experiment can be found in separate folder in supplementary materials.

\subsection{Baseline Hyperparameters}
In our experiments, we employed the following baselines and performed hyperparameter tuning on the validation set. Specifically, each baseline has hyperparameters as follow,
\begin{enumerate}
    \item For MMD, $\alpha \in \{0.01, 0.1, 0.5, 1\}$ controls the weight of regularization.
    \item For CMD, $k \in \{1, 3, 5, 7, 10\}$ determines the number of central moment. $\alpha \in \{0.01, 0.1, 0.5, 1\}$ controls the weight of regularization.
    \item For DANN, $\alpha$ is set in $\{0.1, 0.5, 1\}$ for reverse gradients in backward pass. $\beta \in \{0.01, 0.1, 0.5, 1\}$ controls the weight of regularization.
    \item For CDAN, $\lambda$ is a hyper-parameter between source classifier and conditional domain discriminator. $lo \in \{0.01, 0.1, 1\}$ and $h_i \in \{0.1, 1, 2\}$ are the initial value and final value of $\lambda$. $\beta \in \{0.01, 0.1, 0.5, 1\}$ controls the weight of regularization.
    \item For UDAGCN, the balance parameters $\gamma_1$ and $\gamma_2$ are adjusted carefully in the searching space $\{0.1, 0.3, 0.5, 0.7, 1.0\}$, respectively. The adaptation rate $\lambda$ is the following schedule: $\lambda=\min (\frac{2}{1+\exp (-10 p)}-1,0.1)$, and the $p$ is changing from 0 to 1 within the training process as \cite{udagcn2020}. 
    \item For EERM, we search the best learning rate $\alpha_f \in\{0.0001,0.0002,0.001,0.005,0.01\}$ for GNN backbone, the learning rate $\alpha_g \in\{0.0001,0.001,0.005,0.01\}$ for graph editers, the weight $\beta \in\{0.2,0.5,1.0,2.0,3.0\}$ for combination, the number of edge editing for each node $s \in\{1,5,10\}$, the number of iterations $T \in\{1,5\}$ for inner update before one-step outer update. 
    \item For SRGNN-IW$^{\dag}$, the main hyper parameters in the sampler PPR-S are $\alpha \in \{0.01, 0.1, 0.5, 1\}, \gamma \in \{10, 50, 100, 200, 500\}$. When the graph is large, $\epsilon=0.001$ is set in the local algorithm for sparse PPR approximation. $\lambda\in \{0.1, 0.5, 1, 2\}$ is the penalty parameter for the discrepancy regularizer. The lower bound for the instance weight $B_l$ is in $\{0.1, 0.2, 0.5, 1.0\}$.
    \item Hyperparameters of \Ours $\alpha$ and $\beta$ are selected between $\{0.01, 0.1, 1\}$.
    
\end{enumerate}

\section{Experiment Details}
\label{supp:exp}

\begin{table*}[t]
\caption{Dataset Statistics.}
\label{tab:data-stats}
\centering
\scalebox{0.7}{
\begin{tabular}{l c c c c c c c c c c c c c}
\toprule
& \texttt{syn-csbm} & \texttt{syn-cora} & \texttt{syn-products} & \texttt{cora} & \texttt{citeseer} & \texttt{pubmed} & \texttt{DBLP} & \texttt{ACM}  & \texttt{BACE} & \texttt{BBBP} & \texttt{Clintox} 

 \\ \midrule
\# Graphs & 500 & 30 & 30 & 1 & 1 & 1 & 2 & 2 & 1513 & 2039 & 1478\\
\# Nodes & 128 & 1,490 & 10,000 & 2,708 & 3,327 & 19,717 & 78,509 & 23,343 & 34 & 24 & 26\\
\# Edges & 1,280 & 2,965 & 59,640 & 5,278 & 4,614 &44,325 & 1,001,300 & 162,106 & 74 & 52 & 56\\
\# Classes & 2 & 5 & 10 & 7 & 6 & 3 & 5 & 5 & 2 & 2 & 2\\
\bottomrule
\end{tabular}
}
\end{table*}

\subsection{Dataset Details}
In the main paper, we perform node classification and graph classification tasks on 11 different datasets with distribution shift. The statistics of these graphs are presented in Table \ref{tab:data-stats}. We will now discuss the selection criterion or creation process for each dataset in detail.
\begin{figure}
    \centering
    \includegraphics[width=0.4\linewidth]{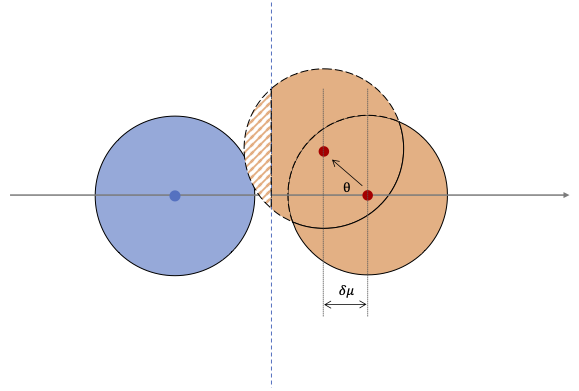}
    \caption{Illustration of creating feature shift on CSBM graphs.}
    \label{fig:creat_figure_shift}
\end{figure}

\xhdr{CSBM Dataset Generation.}
In our experiments, we set the feature size $d$ and average degree $D$ of  CSBM graph in Definition 3.3 graph as 128 and 10, respectively.

For structure shift (\ie \texttt{syn-csbm-pq} ), each time we first sample a feature mean $\mu \sim \mathcal{N}(0, \frac{1}{\sqrt{d}})$, where $d$ is the dimension of the feature.
Then source graph is generated with a fixed $p/q=5$ while each target graph is generated under a random $p/q$ between $\{1,...,10\}$. Such that we ensure the features of both graph are generated with the same Gaussian distribution and their homophily ratios are different.

For feature shift (\ie \texttt{syn-csbm-$\delta$}), we generate $\mu'$ by translating mean by $\delta \mu$ and rotate $\mu'$ by $\theta$ (from 0 to 60 degrees). In corollary 3.1.2, we use the same $\delta$ to describe the classification error. When $\delta$ is small, feature shift is small and test feature mean $\mu'$ is close to original feature mean. The rotation is added to avoid trivial adaptation like translation. Figure~\ref{fig:creat_figure_shift} illustrates the process of creating features shifts in our experiment.
The dataset generation code can be found in uploaded code named cSBM\_gendata.py.

\xhdr{DBLP-ACM Dataset.}
In the main paper, we conduct the transfer learning experiments with \emph{domain shift} and \emph{time shift} for node classification. These experiments use three sets of citation networks, which are constructed on the datasets provided by ArnetMiner \cite{tang2008arnetminer}. 
Specifically, for domain shift, we adopt two sets of ACM-DBLP citation networks of different sizes. 
The small set namely ACM-DBLP$_\text{small}$ is proposed by \cite{udagcn2020}. 
It includes the papers extracted from ACMv9 (between years 2000 and 2010) and DBLPv8  (after year 2010). 
The large set, ACM-DBLP$_\text{large}$ is constructed on DBLPv12 (before 2017) and ACMv8 (before 2017). 
As to time shift, we utilize ACMv9 across different time periods, specifically, before or after 2010, to build two citation networks, ACM$_\text{time}$. 
In our experiments, we consider these datasets as undirected graphs and each edge representing a citation relation between two papers. 
The papers are classified to some of the predefined categories according to its research topics. 
ACM-DBLP$_\text{small}$ has six categories including“Database”, “Data mining”, “Artificial intelligent”, “Computer vision”, “Information Security” and "High Performance Computing".  
For ACM-DBLP$_\text{large}$ and ACM$_\text{time}$, there are five categories including “Database”, “Data mining”, “Artificial intelligent”, “Computer vision”, and “Natural Language Processing".  
We evaluate our proposed methods by conducting multi-label classification on these three sets of citation networks.

\xhdr{Graph Classification Datasets.} There are 10 molecular propety prediction datasets from Open Graph Benchmark~\cite{hu2020open}. These graphs are known to be affected by the scaffold split of the training and testing data. To compare different domain adaptation algorithms, we rank the performance degradation by comparing validation and test accuracy. From Table 2 in the main paper, we select the top-3 datasets with the highest degradation: BACE, BBBP, and Clintox. We choose these datasets because they exhibit the most pronounced "negative" distribution shifts.

\subsection{Complementary Results on Synthetic Domain Adaptation}
In Figure~\ref{supp:csbm_adaptation}, we provide the test logloss plot of our experiments on CSBM graphs as complimentary results of Figure 3 of the main paper, respectively. The test loss also correlates well with domain adaptation bound introduced in Theorem 4.1. When distribution shift becomes more significant, for example a smaller p/q or larger $\delta$, the target loss increases.
In addition, we present the numerical results used to draw Figure 3a and 3b of the main paper in Table~\ref{tab:syn_csbm_p_q} and Table~\ref{tab:syn_csbm_delta}.
\begin{figure}[h]
\centering

\begin{subfigure}{0.45\linewidth}
\includegraphics[width=0.9\textwidth]{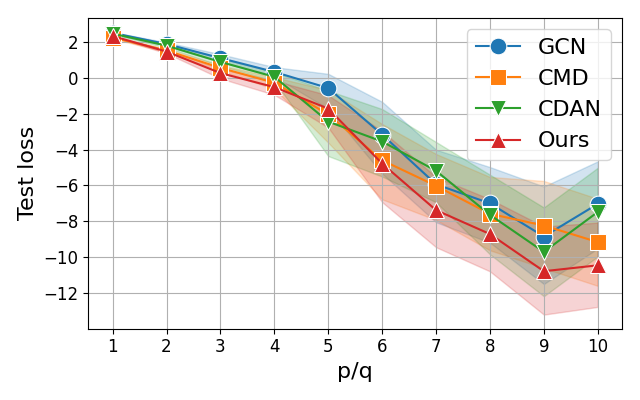}
\centering
\caption{Testing loss of different DA algorithms.}
\end{subfigure}
\hspace*{\fill}
\begin{subfigure}{0.45\linewidth}
\includegraphics[width=0.9\textwidth]{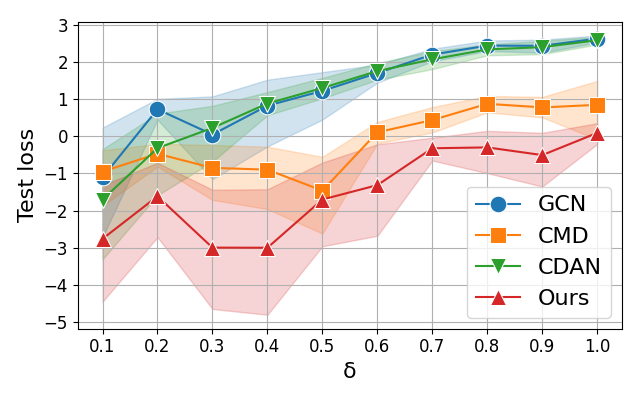}
\centering
\caption{Testing Loss of different DA algorithms.}
\end{subfigure}

\caption{Domain adaptation on datasets constructed from real graphs. We use homophily ratio $1.0$ for training and plot the base GCN performance as well as domain adaption algorithms on three test graphs per interval.}
\label{supp:csbm_adaptation}
\end{figure}

\begin{table*}[t]
\caption{syn-csbm-p/q (Fig. 3a). Mean ROC and standard deviation per method (with structure shift $p/q$).
}
\centering
\scalebox{0.75}{
\begin{tabular}{l c c c c c c c c c c}
\toprule

{\multirow{2}{*}{Method}} & \multicolumn{10}{c}{\texttt{syn-csbm-$p/q$}}  \\                        & 1 & 2 & 3 & 4 & 5 & 6 & 7 & 8 & 9 & 10    \\
 \midrule
GCN  &   62.6 \scriptsize{$\pm$ 6.3} &  78.7 \scriptsize{$\pm$ 8.2}  &    87.9 \scriptsize{$\pm$ 8.8}   & 93.1 \scriptsize{$\pm$ 8.1}    & 94.9 \scriptsize{$\pm$ 5.7} &    97.1 \scriptsize{$\pm$ 3.6}  &  97.6 \scriptsize{$\pm$ 4.4}   &   98.8 \scriptsize{$\pm$ 1.7}    &   98.4 \scriptsize{$\pm$ 3.0} &  97.8 \scriptsize{$\pm$ 4.6}               \\
CMD &    66.0 \scriptsize{$\pm$ 5.0} &  83.7 \scriptsize{$\pm$ 4.0}  &    93.1 \scriptsize{$\pm$ 3.3}   & 96.2 \scriptsize{$\pm$ 2.6}    & 97.9 \scriptsize{$\pm$ 1.5} &    98.5 \scriptsize{$\pm$ 1.4}  &  98.8 \scriptsize{$\pm$ 1.4}   &   99.1 \scriptsize{$\pm$ 1.2}    &   99.3 \scriptsize{$\pm$ 0.9} &  99.3 \scriptsize{$\pm$ 1.0}   \\
CDAN  &  62.7 \scriptsize{$\pm$ 5.9} &  79.2 \scriptsize{$\pm$ 8.1}  &    90.0 \scriptsize{$\pm$ 6.5}   & 94.6 \scriptsize{$\pm$ 5.4}    & 96.0 \scriptsize{$\pm$ 4.7} &    97.9 \scriptsize{$\pm$ 2.0}  &  98.4 \scriptsize{$\pm$ 3.6}   &   99.1 \scriptsize{$\pm$ 1.1}    &   99.1 \scriptsize{$\pm$ 1.4} &  98.6 \scriptsize{$\pm$ 2.7}             \\
Ours  &  68.1 \scriptsize{$\pm$ 5.4} &  85.9 \scriptsize{$\pm$ 4.0}  &    94.7 \scriptsize{$\pm$ 3.0}   & 96.9 \scriptsize{$\pm$ 2.1}    & 98.4 \scriptsize{$\pm$ 1.3} &    98.9 \scriptsize{$\pm$ 1.0}  &  99.4 \scriptsize{$\pm$ 0.7}   &   99.5 \scriptsize{$\pm$ 0.5}    &   99.7 \scriptsize{$\pm$ 0.5} &  99.6 \scriptsize{$\pm$ 0.5}  \\

\bottomrule
\end{tabular}
}
\label{tab:syn_csbm_p_q}
\end{table*}

\begin{table*}[t]
\caption{syn-csbm-$\delta$ (Fig. 3b). Mean ROC and standard deviation per method (with feature shift $\delta$).
}
\centering
\scalebox{0.7}{
\begin{tabular}{l c c c c c c c c c c}
\toprule
{\multirow{2}{*}{Method}} & \multicolumn{10}{c}{\texttt{syn-csbm-$\delta$}}  \\                        & 0.1 & 0.2 & 0.3 & 0.4 & 0.5 & 0.6 & 0.7 & 0.8 & 0.9 & 1.0    \\
 \midrule
GCN      &   91.5 \scriptsize{$\pm$ 10.7}    &  89.6 \scriptsize{$\pm$ 9.9} & 87.9 \scriptsize{$\pm$ 11.2} &   82.9 \scriptsize{$\pm$ 14.3} & 80.2 \scriptsize{$\pm$ 13.8} &   78.1 \scriptsize{$\pm$ 13.6} & 69.1 \scriptsize{$\pm$ 12.8} &   61.8 \scriptsize{$\pm$ 13.1} & 61.1 \scriptsize{$\pm$ 13.4} &   56.8 \scriptsize{$\pm$ 10.2}   \\
CMD     &   97.5 \scriptsize{$\pm$ 1.9} &  97.0 \scriptsize{$\pm$ 1.8}  &    96.9 \scriptsize{$\pm$ 2.6}   & 97.1 \scriptsize{$\pm$ 2.5}    & 96.2 \scriptsize{$\pm$ 5.1} &    94.2 \scriptsize{$\pm$ 5.5}  &  89.5 \scriptsize{$\pm$ 18.5}  &   87.3 \scriptsize{$\pm$ 15.5}  &    87.2 \scriptsize{$\pm$ 19.2}  & 80.1 \scriptsize{$\pm$ 20.2}  \\
CDAN    &   93.5 \scriptsize{$\pm$ 7.5} &  90.2 \scriptsize{$\pm$ 9.3}  &    87.1 \scriptsize{$\pm$ 11.4}  & 84.4 \scriptsize{$\pm$ 12.6}  &  79.0 \scriptsize{$\pm$ 13.5}  &   72.6 \scriptsize{$\pm$ 12.4}  &    66.6 \scriptsize{$\pm$ 12.1}  & 60.8 \scriptsize{$\pm$ 13.0}  &  59.4 \scriptsize{$\pm$ 12.2}  &   55.5 \scriptsize{$\pm$ 8.0}          \\
Ours  &  98.1 \scriptsize{$\pm$ 1.5} &  97.8 \scriptsize{$\pm$ 1.5}  &    98.0 \scriptsize{$\pm$ 1.8}   & 98.1 \scriptsize{$\pm$ 1.5}    & 97.4 \scriptsize{$\pm$ 4.1} &    96.9 \scriptsize{$\pm$ 1.8}  &  96.3 \scriptsize{$\pm$ 2.4}   &   95.3 \scriptsize{$\pm$ 3.3}    &   95.2 \scriptsize{$\pm$ 3.9} &  94.0 \scriptsize{$\pm$ 5.4}  \\
\bottomrule
\end{tabular}}
\label{tab:syn_csbm_delta}
\end{table*}

\subsection{Complementary Results on Supervised Node Classification}
Due to the space limit, we only report the Micro-F1 in the Table 2 of the main paper. In Table~\ref{tab:ood_test_classification-result}, we include the results on both Micro-F1 and Macro-F1.

\begin{table*}[t]
\caption{Full result of supervised node classification. We report mean and standard deviation on Micro and Macro F1.
}
\label{tab:ood_test_classification-result}
\centering
\scalebox{1}{
\begin{tabular}{ll|c c c c c c}
\toprule
\multicolumn{2}{c }{\multirow{2}{*}{Method}}                             & \multicolumn{2}{c}{$\text{ACM-DBLP}_\text{small}$}            & \multicolumn{2}{c}{$\text{ACM}_\text{time}$} & \multicolumn{2}{c}{$\text{ACM-DBLP}_\text{large}$}     \\
\multicolumn{2}{c}{}                                                    & Micro-F1 & Macro-F1  & Micro-F1 & Macro-F1 & Micro-F1 & Macro-F1   \\ \midrule
\multicolumn{2}{l}{Base model}   &  68.1 \scriptsize{$\pm$ 2.1} &    68.2 \scriptsize{$\pm$ 2.4}	 &  78.8 \scriptsize{$\pm$ 1.0} &    76.1 \scriptsize{$\pm$ 0.7}	 &  81.1 \scriptsize{$\pm$ 0.2} &    79.1 \scriptsize{$\pm$ 0.2}                \\
\multicolumn{2}{l}{MMD} &   65.9 \scriptsize{$\pm$ 2.2}	&   65.3 \scriptsize{$\pm$ 3.1}		&   79.0 \scriptsize{$\pm$ 1.0}	&   76.1 \scriptsize{$\pm$ 1.0}		&   81.7 \scriptsize{$\pm$ 0.3}	&   79.6 \scriptsize{$\pm$ 0.3}   \\
\multicolumn{2}{l}{CMD$^{\dag}$}        &   75.5 \scriptsize{$\pm$ 4.4}	    &   71.9 \scriptsize{$\pm$ 6.8}		    &   79.4 \scriptsize{$\pm$ 0.7}	    &   75.9 \scriptsize{$\pm$ 0.7}		    &   75.2 \scriptsize{$\pm$ 0.8}	    &   74.7 \scriptsize{$\pm$ 0.7}             \\

\multicolumn{2}{l}{DANN}  & 70.1 \scriptsize{$\pm$ 1.8}	& 70.5 \scriptsize{$\pm$ 1.7}		& 79.6 \scriptsize{$\pm$ 0.4}	& 76.9 \scriptsize{$\pm$ 0.4}		& 81.6 \scriptsize{$\pm$ 0.4}	& 80.0 \scriptsize{$\pm$ 0.4}  \\

\multicolumn{2}{l}{CDAN}                   &  75.3 \scriptsize{$\pm$ 4.3} & 	75.2 \scriptsize{$\pm$ 4.6} & 		79.3 \scriptsize{$\pm$ 1.3} & 	76.4 \scriptsize{$\pm$ 0.9} & 		82.1 \scriptsize{$\pm$ 0.3} & 	80.0 \scriptsize{$\pm$ 0.2}    \\
\midrule

\multicolumn{2}{l}{UDAGCN}      & 66.4 \scriptsize{$\pm$ 5.1} & 	64.1 \scriptsize{$\pm$ 6.2} & 		79.3 \scriptsize{$\pm$ 0.5} & 	74.6 \scriptsize{$\pm$ 0.4} & 		78.3 \scriptsize{$\pm$ 2.6} & 	74.5 \scriptsize{$\pm$ 2.7}   \\
\multicolumn{2}{l}{EERM}            & 64.9 \scriptsize{$\pm$ 3.5} & 	60.0 \scriptsize{$\pm$ 3.2} & 		77.3 \scriptsize{$\pm$ 0.4} & 	74.5 \scriptsize{$\pm$ 0.3} & 		81.0 \scriptsize{$\pm$ 0.4} & 	78.1 \scriptsize{$\pm$ 0.4}            \\
\multicolumn{2}{l}{SRGNN-IW$^{\dag}$}                 & 69.2 \scriptsize{$\pm$ 1.6} & 	69.9 \scriptsize{$\pm$ 1.7} & 		79.5 \scriptsize{$\pm$ 1.1} & 	76.7 \scriptsize{$\pm$ 0.8} & 		81.4 \scriptsize{$\pm$ 0.4} & 	79.5 \scriptsize{$\pm$ 0.3}  \\

  \midrule
\multicolumn{2}{l}{\Ours($\alpha=0$)}             & 74.0 \scriptsize{$\pm$ 4.7} & 	73.3 \scriptsize{$\pm$ 4.9} & 		80.1 \scriptsize{$\pm$ 0.5} & 	77.2 \scriptsize{$\pm$ 0.4} & 		82.1 \scriptsize{$\pm$ 0.3} & 	80.0 \scriptsize{$\pm$ 0.3}        \\
\multicolumn{2}{l}{\Ours($\beta=0$)}                  & 71.6 \scriptsize{$\pm$ 2.3} & 	71.2 \scriptsize{$\pm$ 2.6} & 		80.2 \scriptsize{$\pm$ 0.4} & 	77.3 \scriptsize{$\pm$ 0.3} & 		82.3 \scriptsize{$\pm$ 0.4} & 	80.2 \scriptsize{$\pm$ 0.4}        \\
\multicolumn{2}{l}{\Ours}                 &  78.5 \scriptsize{$\pm$ 4.0} & 	78.1 \scriptsize{$\pm$ 4.3} & 		80.3 \scriptsize{$\pm$ 0.8} & 	77.3 \scriptsize{$\pm$ 0.6} & 		82.5 \scriptsize{$\pm$ 0.3} & 	80.4 \scriptsize{$\pm$ 0.3}  \\
\bottomrule
\end{tabular}}
\label{tab:small_network}
\end{table*}

\subsection{Hyperparameter and Complexity Study}
\xhdr{Choices of $\alpha$ and $\beta$.}
The main difference between \Ours and \Ours++ is the introduction of aligning marginal distribution $(\mu_\S^g,\mu_\T^g)$ together with conditional shift controlled by hyper-parameter $\alpha$ in Eq.(11) of the main paper.
In this section we study how varying $\alpha,\beta$ between $[0,1]$ in \Ours++ affects the performance. We conduct 10 runs for each $\alpha$ while fixing $\beta=0.1$ on four node classification datasets and vice versa. In Figure~\ref{fig:varying_alpha}, we observed that \Ours++ does not consistently outperform \Ours ($\alpha=0$) except dataset ACM-DBLP. Because different domains may have different word distributions as node features, and in this case we find that regularizing the representation shift appears to be helpful. 
In Figure~\ref{fig:varying_beta}, we observe that the performance on all four datasets improves when $\beta>0$, further validating that minimizing conditional shift is a key factor in our framework. Overall, our performance is not sensitive to the hyper parameters within a reasonable range.

\begin{figure}
\begin{subfigure}{0.45\linewidth}
\includegraphics[width=1\textwidth]{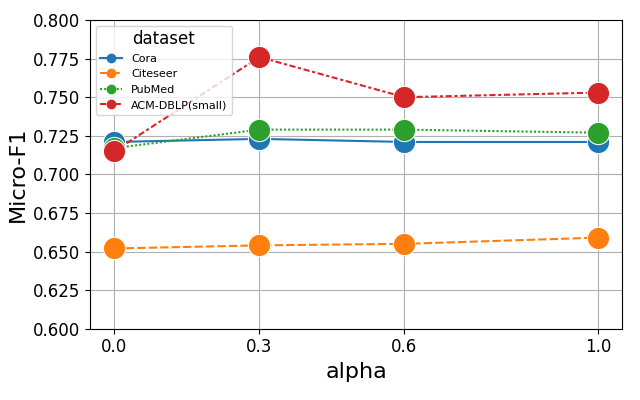}
\centering
\caption{Node classification accuracy varying $\alpha$}
\label{fig:varying_alpha}
\end{subfigure}
\hspace*{\fill}
\begin{subfigure}{0.45\linewidth}
\includegraphics[width=1\textwidth]{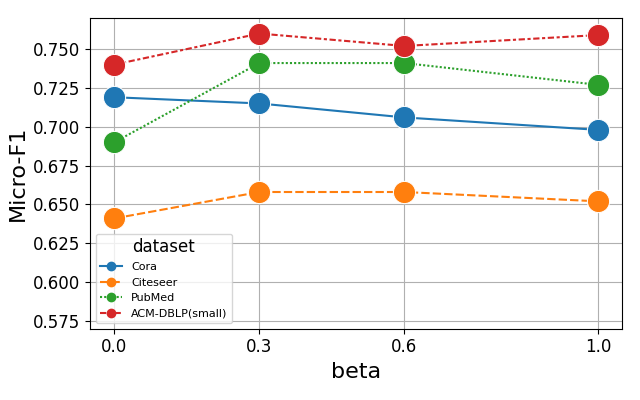}
\centering

\caption{Node classification accuracy varying $\beta$}
\label{fig:varying_beta}
\end{subfigure}
\end{figure}

\xhdr{Time and Space Complexity of \Ours.} We would like to provide further details on training time and extra costs on a non-citation graph from Open Graph Benchmark~\cite{hu2020open} - ogbn-proteins. In ogbn-proteins, nodes represent proteins, and edges indicate different types of biologically meaningful associations between proteins. The task is to predict the presence of protein functions in a multi-label binary classification setup, where there are 112 kinds of labels to predict in total. It is considered as reasonably large with 132 thousand nodes and 39 million edges. We report the actual running time and actual GPU usage per epoch varying batch size $N$ in Table~\ref{tab:time_space_complexity}. We observe that the training time of \Ours increases only slightly when the batch size is set to 128. The additional space complexity is negligible for all batch sizes. The additional time complexity, as explained earlier, is primarily influenced by the batch size. Choosing an appropriate batch size, such as 128 or 256, can reduce the computation cost of solving the optimal transportation plan in \Ours.

\begin{table}[]
    \centering
    \caption{Additional Time and Space Complexity of \Ours.}
    \label{tab:time_space_complexity}
    
    \begin{tabular}{c*{6}{c}}
    \toprule
    \multirow{2}{*}{Method} & \multicolumn{3}{c}{Time} & \multicolumn{3}{c}{Space} \\
    & 128 & 256 & 512 & 128 & 256 & 512 \\
    \midrule
    GraphSAGE & 6min04s & 6min20s & 6min51s & 5035MB & 5075MB & 5149MB \\
    \Ours & 6min46s & 8min42s & 14min08s & 5081MB & 5129MB & 5389MB \\
    \bottomrule
    \end{tabular}
    
\end{table}

\subsection{Additional Experiments on GraphOOD Benchmark}
We performed additional experiments on graph classification using six datasets obtained from the data curators of DrugOOD~\cite{DBLP:conf/nips/YangZWJY22}. 
The DrugOOD dataset is derived from the ChEMBL website, which houses a large-scale bioassay deposition~\cite {DBLP:journals/nar/MendezGBCVFMMMN19}. 
The dataset offers various indicators for splitting, such as assay, scaffold, and size. 
Furthermore, we applied three different splitting schemes to both IC50 and EC50 categories in DrugOOD. 
As a result, we obtained six datasets: EC50-$\star$ and IC50-$\star$, where the suffix $\star$ denotes the specific splitting scheme (IC50/EC50-assay/scaffold/size). 
This approach enables us to comprehensively evaluate the performance of our method under different environmental definitions.
All six datasets focus on ligand-based affinity prediction (LBAP), where each molecule is labeled as active or inactive. 
For all datasets, we followed the default training-validation-test split outlined following ~\cite{DBLP:conf/nips/YangZWJY22}. 
During training, we utilized all molecules in the training set to optimize the model parameters. 
Subsequently, we selected hyperparameters based on the validation set and reported the results on the test molecule set using the model that achieved the best performance on the validation set.

For graph classification, to build the base model, we adopt a 4-layer GIN~\cite{DBLP:conf/iclr/XuHLJ19} for node representations and a mean pooling layer for graph representations followed by a linear head to make prediction. 
The experimental results are presented in Table \ref{tab:graph_classification_drugood}. 
Upon the careful observations, we can find several noteworthy discoveries. 
Firstly, it becomes evident that the performance of different algorithms varies significantly across different settings, predominantly due to the presence of distinct distribution shifts. 
This implies that algorithm selection should be tailored to the specific characteristics of the dataset and the nature of the distribution shift. 
In addition, our proposed approach (\Ours) and its variants consistently outperform the other baseline methods. 
This persistent superiority can be attributed to the deliberate design of our approach, which prioritizes optimal performance in graph classification scenarios. 
The underlying techniques and mechanisms employed by our approach effectively leverage the inherent structure and relationships within graph nodes, leading to superior classification accuracy. 
Furthermore, the standout performance of our approach (\Ours++) should not be overlooked. 
Across all datasets, \Ours++ consistently achieved a top ranking, showcasing its robustness and effectiveness. 
This consistent high performance across various datasets signifies the potential of our approach to accurately predict graph properties and opens up promising avenues for its practical applications.

\renewcommand\arraystretch{1.0}
\begin{table}[t]

\center \footnotesize
\tabcolsep0.22 in
\caption{Experimental results for graph classification on Drugood datasets (including Accurary and AUC scores).}
\label{tab:graph_classification_drugood}
\begin{tabular}{lcccccc}
\toprule
\multicolumn{1}{l}{\multirow{3}[1]{*}{\textbf{Model}}} & \multicolumn{6}{c}{\textbf{lbap\_core\_ec50}} \\
\cmidrule(lr){2-7}
 & \multicolumn{2}{c}{\textbf{Assay}} & \multicolumn{2}{c}{\textbf{Scaffold}} & \multicolumn{2}{c}{\textbf{Size}} \\
 \cmidrule(lr){2-3}  \cmidrule(lr){4-5} \cmidrule(lr){6-7}
 & \textbf{ACC} & \textbf{AUC} & \textbf{ACC} & \textbf{AUC} & \textbf{ACC} & \textbf{AUC} \\
\midrule
Base Model & 87.89 & 69.46 & 70.32 & 59.66 & 67.86 & 61.53 \\
\midrule
CMD & 70.68 & 50.15 & 58.51 & 47.99 & 67.22 & 57.97 \\
DANN & 87.49 & 64.70 & 68.36 & 58.16 & 67.78 & 47.58 \\
CDAN & 87.54 & 69.25 & 70.88 & 60.23 & 68.14 & 61.10 \\
\midrule
UDAGCN & 82.79 & 73.05 & 72.38 & 61.23 & 69.70 & 61.06 \\
SRGNN-IW & 87.39 & 74.04 & 72.05 & 60.35 & 68.94 & 60.33 \\
\midrule
\Ours  & 88.62 & \textbf{74.40} & 71.94 & 60.45 & 69.42 & 59.94 \\
\Ours++   & \textbf{88.84} & 71.22 & \textbf{73.57} & \textbf{61.75} & \textbf{70.34} & \textbf{61.57} \\
\midrule
\multicolumn{1}{l}{\multirow{3}[1]{*}{\textbf{Model}}} & \multicolumn{6}{c}{\textbf{lbap\_core\_ic50}} \\
\cmidrule(lr){2-7}
 & \multicolumn{2}{c}{\textbf{Assay}} & \multicolumn{2}{c}{\textbf{Scaffold}} & \multicolumn{2}{c}{\textbf{Size}} \\
 \cmidrule(lr){2-3}  \cmidrule(lr){4-5} \cmidrule(lr){6-7}
 & \textbf{ACC} & \textbf{AUC} & \textbf{ACC} & \textbf{AUC} & \textbf{ACC} & \textbf{AUC} \\
\midrule
Base model & 81.21 & 68.34 & 74.04 & 63.69 & 72.80 & 61.51 \\
\midrule
CMD & 74.24 & 68.55 & 72.54 & 60.33 & 68.30 & 58.14 \\
DANN & 83.22 & 70.08 & 76.00 & 66.37 & 70.08 & 63.45 \\
CDAN & 83.06 & 71.29 & 76.52 & 66.42 & 72.87 & 64.79 \\
\midrule
UDAGCN & 81.34 & 69.89 & 74.66 & 63.77 & 72.96 & 64.79 \\
SRGNN-IW & 82.91 & 71.00 & 75.51 & 63.80 & 73.32 & 64.85 \\
\midrule
\Ours  & 83.47 & \textbf{72.40} & \textbf{77.77} & \textbf{67.50} & 73.42 & 62.50 \\
\Ours++  & \textbf{83.56} & 71.64 & 77.36 & 66.04 & \textbf{73.92} & \textbf{65.87} \\
\bottomrule
\end{tabular}
\end{table}
\end{document}